\documentclass{article}

\usepackage{microtype}
\usepackage{graphicx}
\usepackage{booktabs} 
\usepackage{bm}
\usepackage{hyperref}
\usepackage{subcaption}
\usepackage{graphicx}
\usepackage{xspace}
\usepackage{amsmath}
\usepackage{makecell}
\usepackage{multirow}
\usepackage{amsfonts}
\usepackage{pythonhighlight}
\usepackage{amsthm}

\newtheorem{definition}{Definition}[section]
\newtheorem{theorem}{Theorem}
\newtheorem{lemma}[theorem]{Lemma}

\newcommand{\ie}{\emph{i.e.}\xspace} 
\newcommand{\eg}{\emph{e.g.}\xspace} 

\usepackage[accepted]{icml2020}

\icmltitlerunning{SDE-Net: Equipping Deep Neural Networks with Uncertainty Estimates}

\begin{document}

\twocolumn[
\icmltitle{SDE-Net: Equipping Deep Neural Networks with Uncertainty Estimates}

\begin{icmlauthorlist}
\icmlauthor{Lingkai Kong}{GT}
\icmlauthor{Jimeng Sun}{UIUC}
\icmlauthor{Chao Zhang}{GT}
\end{icmlauthorlist}

\icmlaffiliation{GT}{School of Computational Science and Engineering, Georgia Institute of Technology, Atlanta,  GA}
\icmlaffiliation{UIUC}{Department  of  Computer   Science,  University  of  Illinois at Urbana-Champaign, Urbana, IL}

\icmlcorrespondingauthor{Lingkai Kong}{lkkong@gatech.edu}
\icmlcorrespondingauthor{Chao Zhang}{chaozhang@gatech.edu}

\icmlkeywords{Machine Learning, ICML}

\vskip 0.3in
]



\printAffiliationsAndNotice{}  

\begin{abstract}
    Uncertainty quantification is a fundamental yet unsolved problem for deep
    learning. The Bayesian framework provides a principled way of uncertainty
    estimation but is often not scalable to modern deep neural nets (DNNs) that
    have a large number of parameters. Non-Bayesian methods are simple to implement
    but often conflate different sources of uncertainties and require huge
    computing resources.  We propose a new method for quantifying uncertainties of
    DNNs from a dynamical system perspective.  The core of our method is to view
    DNN transformations as state evolution of a stochastic dynamical system and
    introduce a Brownian motion term for capturing epistemic uncertainty. Based on this
    perspective, we propose a neural stochastic differential equation model
    (SDE-Net) which consists of (1) a drift net that controls the system to fit the
    predictive function; and (2) a diffusion net that captures epistemic uncertainty.
    We theoretically analyze the existence and uniqueness of the solution to
    SDE-Net. Our experiments demonstrate that the SDE-Net model can outperform
    existing uncertainty estimation methods across a series of tasks where
    uncertainty plays a fundamental role.
    
    \end{abstract}

    \section{Introduction}

    \begin{figure*}[t]
        \centering
        \begin{subfigure}[b]{0.3\linewidth}
          \includegraphics[width=\linewidth]{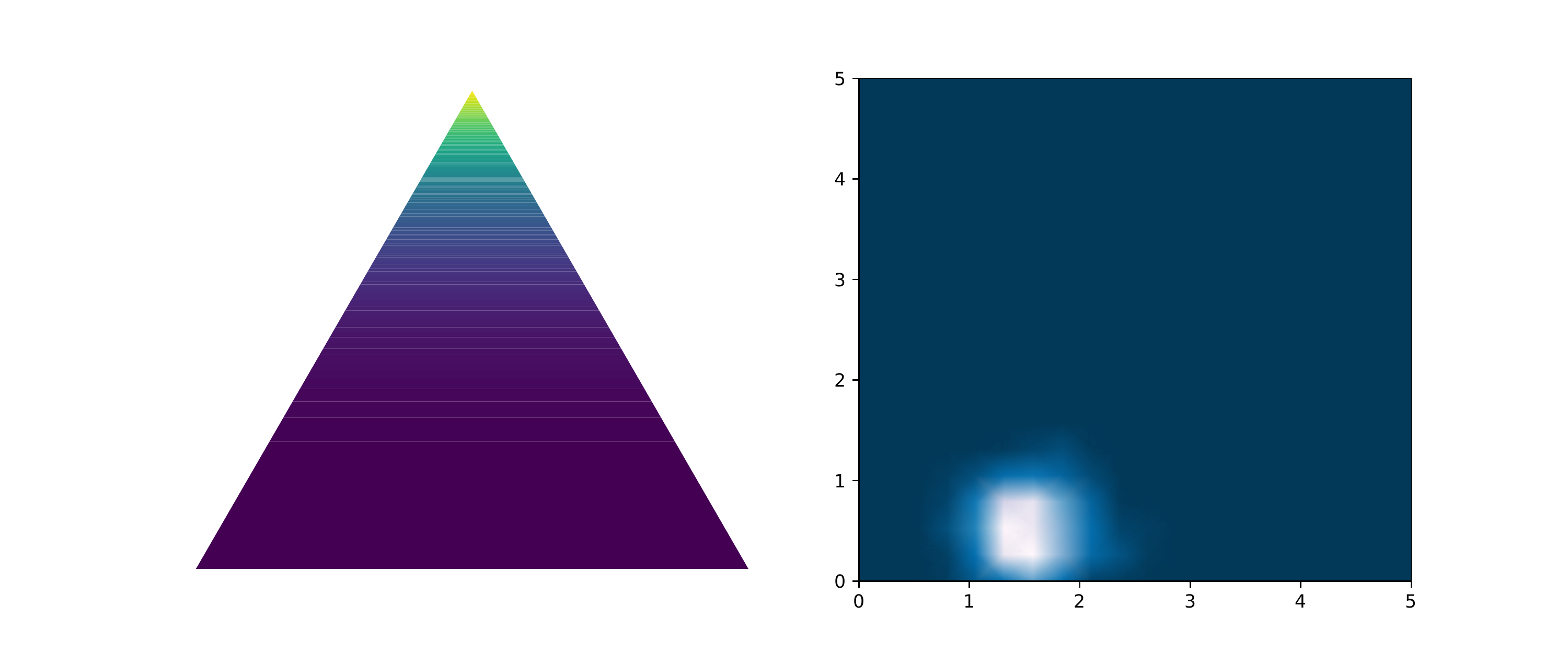}
          \caption{Low aleatoric uncertainty, low epistemic uncertainty.}
          \label{fig:a}
        \end{subfigure}
        \hfil
        \vspace{0.5em}
        \begin{subfigure}[b]{0.3\linewidth}
          \includegraphics[width=\linewidth]{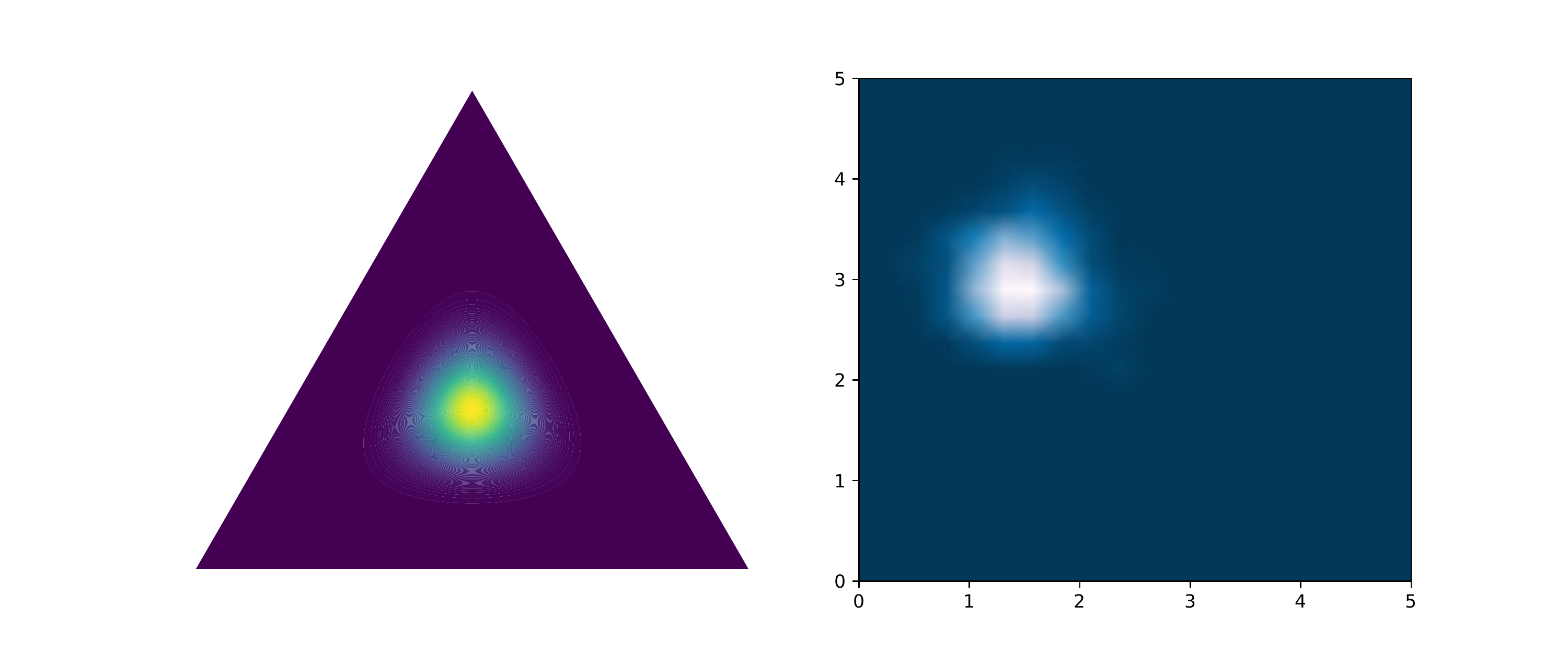}
          \caption{High aleatoric uncertainty, low epistemic uncertainty.}
          \label{fig:b}
        \end{subfigure}
        \hfil
        \begin{subfigure}[b]{0.3\linewidth}
          \includegraphics[width=\linewidth]{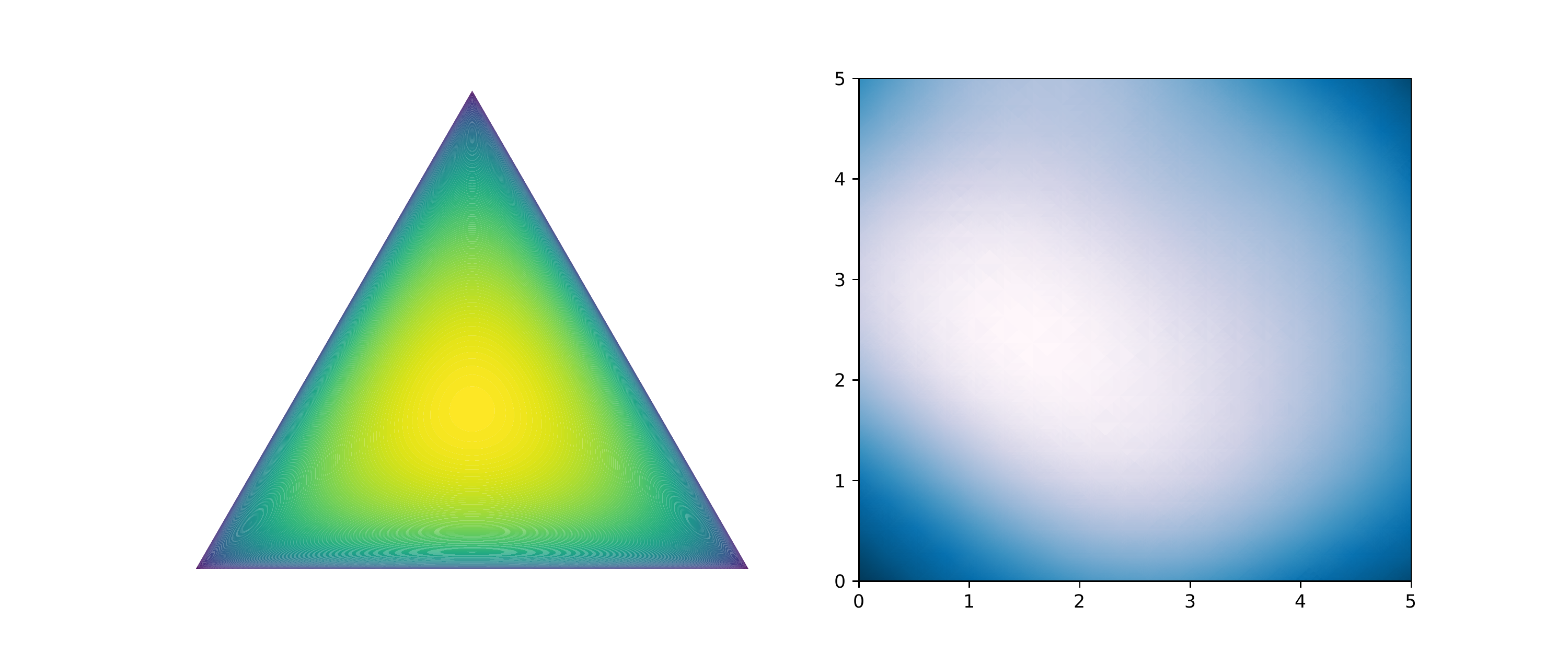}
          \caption{High aleatoric uncertainty, high epistemic uncertainty.}
          \label{fig:c}
        \end{subfigure} 
        \vspace{-3ex}
        \caption{Different behaviors of a probabilistic model under aleatoric and
          epistemic uncertainties for classification and regression tasks.  The heat maps
          represent the distributions of model's predictive distributions.
          The triangles represent classification simplexes and the squares represent regression
          parameter spaces (x-axis is the predictive mean $\mu(\bm{x}^{\ast})$ ;
          y-axis is the predictive variance $\sigma(\bm{x}^{\ast})$). 
        }
        \label{fig:uncertainty-bahavior}
      \end{figure*}

    Deep Neural Nets (DNNs) have achieved enormous success in a wide spectrum of
    tasks, such as image classification \citep{2012:alex}, machine translation
    \citep{2014:Choukroun}, and reinforcement learning \citep{2018:Li}. Despite
    their remarkable predictive performance, DNNs are poor at quantifying
    uncertainties for their predictions. Recent studies have shown that DNNs are
    often overconfident in their predictions and produce mis-calibrated output
    probabilities for classification \citep{Guo:2017}. Moreover, they can make
    erroneous yet wildly confident predictions for out-of-distribution samples that
    are very different from training data \citep{Nguyen_2015_CVPR}. Uncertainty
    quantification, a key component to equip DNNs with the ability of knowing what
    they do not know, has become an urgent need for many real-life applications,
    ranging from self-driving cars to cyber security to automatic medical diagnosis.
    
    Existing approaches to uncertainty quantification for neural nets can be
    categorized into two lines. The first line is based on Bayesian neural nets
    (BNNs) \citep{1991:Denker, 1992:Mackay}. BNNs quantify predictive uncertainty
    by imposing probability distributions over model parameters instead of using
    point estimates.   While BNNs provide a principled way
    of uncertainty quantification, exact inference of parameter posteriors is
    often intractable.  Moreover, specifying parameter
    priors for BNNs is challenging because the parameters of DNNs are huge in size
    and uninterpretable.

    Along another line, several non-Bayesian approaches have been proposed for
    uncertainty quantification. The most prominent idea in this line is model
    ensembling \citep{Lakshminarayanan2016}, which trains multiple DNNs with
    different initializations and uses their predictions for uncertainty
    estimation.  However, training an ensemble of DNNs can be prohibitively
    expensive in practice.  Other non-Bayesian methods \citep{geifman2018biasreduced} suffer from the drawback of conflating
    \textit{aleatoric uncertainty}---the natural randomness inherent in the task,
    with \textit{epistemic uncertainty}---the model uncertainty caused by lack of
    observation data.  In many tasks, it is important to separate these two sources
    of uncertainties.  Taking active learning as an example, one would prefer to
    collect data from regions with high epistemic uncertainty but low aleatoric
    uncertainty \citep{hafner2018reliable}.
    
    We propose a deep neural net model for uncertainty quantification based on
    \emph{neural stochastic differential equation}. Our model, named SDE-Net,
    enjoys a number of benefits compared with existing methods: (1) It explicitly
    models \textit{aleatoric uncertainty} and \textit{epistemic uncertainty} and is
    able to separate the two sources of uncertainties in its predictions; (2) It is
    efficient and straightforward to implement, avoiding the need of specifying
    model prior distributions and inferring posterior distributions as in BNNs; and
    (3) It is applicable to both classification and regression tasks.

    Our model design (Section \ref{sect:method}) is motivated by the connection
    between neural nets and dynamical systems. From the dynamical system
    perspective, the forward passes in DNNs can be viewed as state transformations
    of a dynamic system, which can be defined by an NN-parameterized ordinary
    differential equation (ODE) \citep{Chen2018}. However, neural ODE is
    \emph{deterministic} and cannot capture any uncertainty information. In
    contrast, our model characterizes the transformation of hidden states with
    \emph{stochastic differential equation} (SDE) and adds a Brownian motion term
    to explicitly quantify epistemic uncertainty. Our proposed SDE-Net model thus
    consists of (1) a drift net that parameterizes a differential equation to fit
    the predictive function, and (2) a diffusion net that parameterizes the
    Brownian motion and encourages high diffusion for data outside the training
    distribution.  From a control point of view, the drift net controls the system
    to achieve good predictive accuracy, while the diffusion net characterizes
    model uncertainty in a stochastic environment. We theoretically analyze the
    existence and uniqueness of solution to the proposed stochastic dynamical
    system, which provides insights to design a more efficient and stable network
    architecture. 
    
    Empirical results are presented in Section \ref{sect:exp}. We evaluate four tasks where uncertainty plays a fundamental role: out-of-distribution detection, misclassification detection, adversarial samples detection and active learning. We find that SDE-Net can outperform 
    state-of-the-art uncertainty estimation methods or achieve competitive results across these tasks on various
    datasets.

    \section{Aleatoric Uncertainty and Epistemic Uncertainty} 

    \vspace{0.5em}
    For supervised learning, we are given a training dataset $\mathcal{D}=\{
    \bm{x}_j,y_j\}_{j=1}^N$; we train a model $\mathbf{M}$ parameterized by
    $\bm{\theta}$ and use the model $\mathbf{M}$ to make predictions for any new
    test instance $\bm{x}^{\ast}$. The predictive uncertainty comes from two
    sources \citep{2017:kendall}: \emph{aleatoric uncertainty} and \emph{epistemic
      uncertainty}.  Aleatoric uncertainty represents the natural randomness (\eg,
    class overlap, data noise, unknown factors) inherent in the task and cannot be
    explained away with data; while epistemic uncertainty represents our ignorance
    about model caused by the lack of observation data and is high in regions
    lacking training data.
    \vspace{0.5em}
    
    Figure \ref{fig:uncertainty-bahavior} illustrates the behaviors of a
    probabilistic model under the influence of the two sources of uncertainties:
    (1) When both aleatoric and epistemic uncertainties are low (Figure
    \ref{fig:a}), the model outputs confident predictions with low variance.  This
    makes the output distributions sharply concentrate at a simplex corner (for
    classification) or a small-variance region (for regression); (2) When aleatoric
    uncertainty is high but epistemic uncertainty is low (Figure \ref{fig:b}), the
    predictive distributions concentrate around the simplex center  or
    large-variance regions; (3) When epistemic uncertainty is high (Figure
    \ref{fig:c}), the predictive distributions scatter in a highly diffused way
    over the classification simplex and the regression parameter space.
    
    Bayesian neural networks (BNNs) model epistemic uncertainty by imposing distributions over
    model parameters. They are realized by first specifying prior distributions for
    neural net parameters, then inferring parameter posteriors and further
    integrating over them to make predictions. Unfortunately, such modeling of epistemic uncertainty has two drawbacks.
    First, it is difficult to specify the prior distributions since the parameters of DNNs are uninterpretable.
    Second, exact parameter posterior inference is often intractable due to the large number of parameters in DNNs.  Most approaches for learning
    BNNs fall into one of two categories: variational inference (VI) methods
    \citep{Blundell:2015,louizos2017multiplicative, wu2018deterministic} and Markov chain Monte Carlo (MCMC) methods
    \citep{Welling:2011, Li:2016:PSG:3016100.3016149}.  VI methods require one to
    choose a family of approximating distributions, which may lead to
    underestimation of true uncertainties.  MCMC methods are time-consuming and
    require maintaining many copies of the model parameters, which can be costly
    for large NNs. To overcome such drawbacks, we will propose a more direct and efficient way to model uncertainties.

    \section{Uncertainty Quantification via Neural Stochastic Differential Equation}
    \label{sect:method}
    
    We propose a new uncertainty aware neural net from the
    stochastic dynamical system perspective. The proposed method can distinguish the
    two sources of uncertainties with no need of specifying priors of model parameters and performing complicated Bayesian inference.

    \subsection{Neural Net as Deterministic Dynamical System}
    Our approach relies on the connection between neural nets and dynamic systems,
    which has been investigated in \citep{Chen2018}. As neural nets map an
    input $\bm{x}$ to an output $y$ through a sequence of hidden layers, the hidden
    representations can be viewed as the states of a dynamical system. It is thus
    possible to define a dynamical system by parameterizing its ordinary
    differential equation with a neural net.  To see this, consider the
    transformation between layers in ResNet \citep{2015:He}:
    \begin{equation}
        \bm{x}_{t+1} = \bm{x}_t +f(\bm{x}_t, t),
    \end{equation}
    where $t$ is the index of the layer while $\bm{x}_t$ is the hidden state at layer $t$. We rearrange this equation as $\frac{\bm{x}_{t+\Delta t}-\bm{x}_t}{\Delta t} = f(\bm{x}_t, t) $  where $\Delta t =1$. Letting $\Delta t\to 0$, we obtain:
    \begin{equation}
       \lim_{ \Delta \to 0}  \frac{\bm{x}_{t+\Delta t}-\bm{x}_t}{\Delta t} = \frac{d\bm{x}_t}{dt}=f(\bm{x}_t, t) \Longleftrightarrow d\bm{x}_t = f(\bm{x}_t,t)dt.
    \label{eq:ODE}
    \end{equation}
    The transformations in ResNet can thus be viewed as the discretization of a
    dynamical system, whose continuous dynamics is given by $f(\bm{x}_t,t)$. The
    idea of the neural ODE method \citep{Chen2018} is to parameterize
    $f(\bm{x}_t,t)$ with a neural net and exploit an ODE solver to evaluate the
    hidden unit state wherever necessary.  Such a neural ODE formulation enables
    evaluating hidden unit dynamics with arbitrary accuracy and enjoys better
    memory and parameter efficiency.

    \subsection{Modeling Epistemic Uncertainty with Brownian Motion}
    
    However, neural ODE is a deterministic model and cannot model epistemic
    uncertainty. We develop a neural SDE model to characterize a stochastic
    dynamical system instead of a deterministic one. The core of our neural SDE
    model is to capture epistemic uncertainty with \emph{Brownian motion}, which is
    widely used to model the randomness of moving atoms or molecules in Physics
    \citep{bass_2011}.

    \begin{definition}
    A standard Brownian motion $W_t$ is a stochastic process which satisfies the following properties: a) $W_0=0$; b) $W_t-W_s$ is $\mathcal{N}(0,t-s)$ for all $t \ge s \ge 0$; c) For every pair of disjoint time intervals $[t_1, t_2]$ and $[t_3, t_4]$, with $t_1 < t_2 \le t_3 \le t_4$, the increments $W_{t_4} - W_{t_3}$ and $W_{t_2} - W_{t_1}$ are
    independent random variables.
    \end{definition}

    We add the Brownian motion term into Eq. \eqref{eq:ODE}, which leads to a neural SDE
    dynamical system. The continuous-time dynamics of the system are then expressed as:
    \begin{equation}
        d\bm{x}_t = f(\bm{x}_t, t)dt +g(\bm{x}_t,t)dW_t.
    \end{equation}
    
    Here, $g(\bm{x}_t,t)$ denotes the variance of the Brownian motion and
    represents the epistemic uncertainty for the dynamical system.  This variance
    is determined by which region the system is in. As shown in Fig.~\ref{fig:SDE},
    if the system is in the region with abundant training data and low epistemic
    uncertainty, the variance of the Brownian motion will be small; if the system
    is in the region with scarce training data and high epistemic uncertainty, the
    variance of the Brownian motion will be large.  We can thus obtain an epistemic
    uncertainty estimate from the variance of the final time solution $x_T$. 
    
    \begin{figure}[h]
      \centering
      \begin{subfigure}[t]{0.49\linewidth}
        \includegraphics[width=\linewidth]{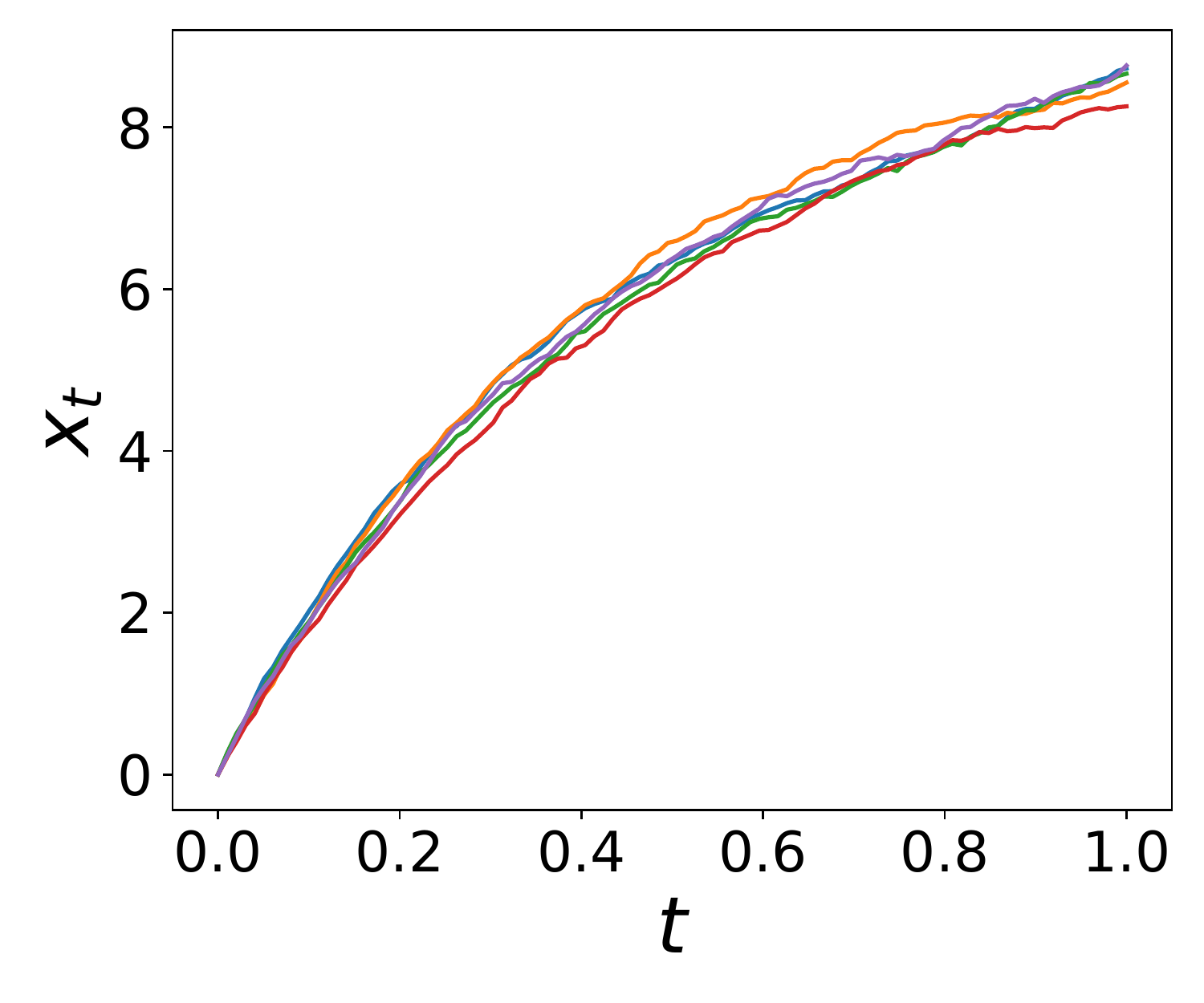}
        \caption{System in the region with low uncertainty}
      \end{subfigure}
      \begin{subfigure}[t]{0.49\linewidth}
        \includegraphics[width=\linewidth]{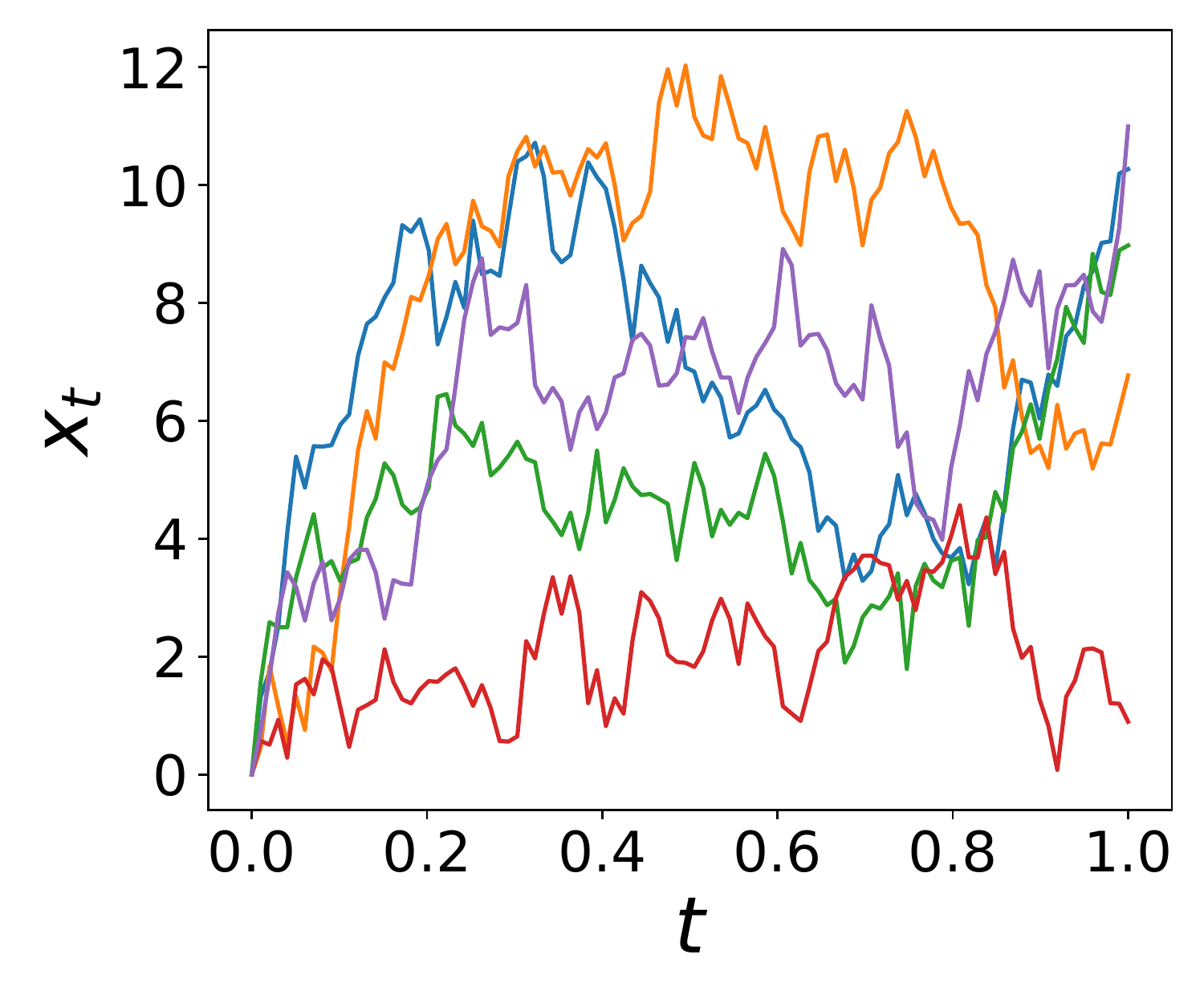}
        \caption{System in the region with high uncertainty}
      \end{subfigure}
    
      \caption{1-D trajectories of a linear SDE for five simulations. When the
      system is in the region with low uncertainty, \ie small $g(\bm{x}_t,t)$,
    the trajectories are more deterministic with small variance.  When the system
    is in the region with high uncertainty,  \ie large $g(\bm{x}_t,t)$, the
    trajectories are more scattered with large variance. }
      \label{fig:SDE}
    \end{figure}

    \subsection{SDE-Net for Uncertainty Estimation}
    
    \begin{figure*}[t]
      \centering
      \includegraphics[width=0.85\linewidth, height=0.35\linewidth]{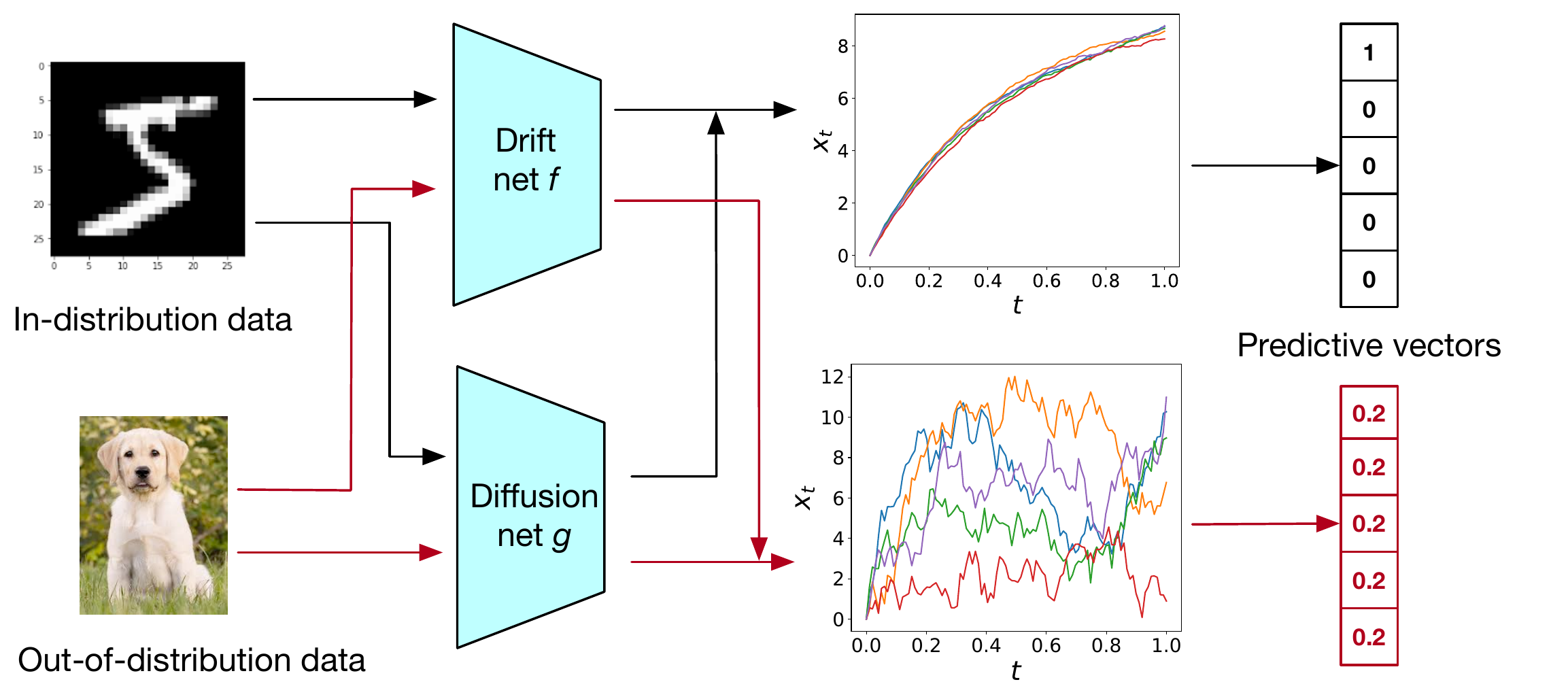}
      \caption{Components of the proposed SDE-Net. For in-distribution data, the system is dominated by the drift net $f$  and achieves good predictive accuracy; for out-of-distribution data, the system is dominated by the diffusion net $g$ and shows high diffusion.}
      \label{fig:illustration}
    \end{figure*}

    As discussed above, we can quantify epistemic uncertainty using Brownian motion.  To make the system able to achieve
    good predictive accuracy and meanwhile provide reliable uncertainty estimates,
    we design our SDE-Net model to use two separate neural nets to represent the
    drift and the diffusion of the system as in Fig.~\ref{fig:illustration}.
    
    The drift net $f$ in SDE-Net aims to control the system to achieve good
    predictive accuracy.  Another important role of the drift net $f$ is to
    capture aleatoric uncertainty. This is achieved by representing model output as
    a probabilistic distribution, \eg, categorical distribution for classification
    and Gaussian distribution for regression.
    
    The diffusion net $g$ in SDE-Net represents the diffusion of the system.  The
    diffusion of the system should satisfy the following: (1) For regions in the
    training distribution, the variance of the Brownian motion should be small (low
    diffusion). The system state is dominated by the drift term in this area and
    the output variance should be small; (2) For regions outside the training
    distribution, the variance of the Brownian motion should be large and the
    system is chaotic (high diffusion).  In this case, the variance of the outputs
    for multiple evaluations should be large.

    Based on the above desired properties, we propose the following objective function for training our SDE-Net model:
    \begin{align}
    &\min_{\bm{\theta}_f} {\rm E}_{\bm{x}_0\sim P_{\text{train}}}{\rm E}(L(\bm{x}_T))+\min_{\bm{\theta}_g}{\rm E}_{\bm{x}_0\sim P_{\text{train}}}g(\bm{x}_0;\bm{\theta}_g)\nonumber\\
    &+\max_{\bm{\theta}_g}{\rm E}_{\tilde{\bm{x}}_0\sim P_{\text{OOD}}}g(\tilde{\bm{x}}_0;\bm{\theta}_g) \nonumber\\
    & s.t. \quad d\bm{x}_t = \underbrace{f(\bm{x}_t,t;\bm{\theta}_f)}_{\text{drift neural net}}dt+\underbrace{g(\bm{x}_0;\bm{\theta}_g)}_{\text{diffusion neural net}}dW_t,
    \label{eq:loss}
    \end{align}
    where $L(\cdot)$ is the loss function dependent on the task, \eg cross entropy loss for classification, $T$ is the terminal time of the
    stochastic process, $P_{\text{train}}$ is the distribution for training data,
    and $P_{\text{OOD}}$ is the out-of-distribution (OOD) data.  To obtain OOD data, we choose to add additive Gaussian noise to obtain noisy
    inputs $\tilde{\bm{x}}_0 = \bm{x}_0+\bm{\epsilon}$ and then distribute the
    inputs according to the convolved distribution as in
    \citep{hafner2018reliable}.  An alternative is to use a different, real dataset
    as a set of samples from the OOD. However, this requires a careful choice of a
    real dataset to avoid overfitting \citep{lee2018training}.

    Unlike the traditional neural nets where each layer
    has its own parameters, the parameters in our proposed SDE-Net are shared by
    each layer.  This can decrease the number of parameters and leads to
    significant memory reduction.
    In the objective function, we also make a simplification that the variance of
    the diffusion term is only determined by the starting point $\bm{x}_0$ instead
    of the instantaneous value $\bm{x}_t$, which is usually sufficient and can make
    the optimization procedure easier.

    \textbf{Uncertainty Quantification}: Once an SDE-Net is learned, we can obtain
    multiple random realizations of the SDE-Net to get samples $\{
    \bm{x}_T\}_{m=1}^M$ and then compute the two uncertainties from them.  The
    aleatoric uncertainty is given by the expected predictive entropy
    $\mathbb{E}_{p(\bm{x}_T|\bm{x}_0,
    \bm{\theta}_{f,g})}[\mathcal{H}[p(y|\bm{x}_T)]]$ in classification and expected
    predictive variance $\mathbb{E}_{p(\bm{x}_T|\bm{x}_0,
    \bm{\theta}_{f,g})}[\sigma(\bm{x}_T)]$ in regression.  The epistemic
    uncertainty is given by the variance of the final solution ${\rm
    Var}(\bm{x}_T)$.  This sampling-and-computing operation shares similar spirit
    with the traditional ensembling method. However, a key
    difference exists between the two: ensembling methods require training multiple
    deterministic NNs, while our method just trains one neural SDE model and uses
    the Brownian motion to encode uncertainty, which incurs much lower time and
    memory costs.

    \subsection{Theoretical Analysis}
    In this subsection, we study the existence and uniqueness of the solution
    $\bm{x}_t (0\le t\le T)$ of the proposed stochastic system. Through this
    theoretical analysis, we can gain insights in designing a more effective
    network architecture for both the drift net $f$ and the diffusion net $g$.
    \begin{theorem}
    When there exists $C >0$ such that 
    \begin{equation}
        \begin{split}
        &||f(\bm{x},t;\bm{\theta}_f)-f(\bm{y},t;\bm{\theta}_f)||+||g(\bm{x};\bm{\theta}_g)-g(\bm{y};\bm{\theta}_g)|| \\
        &\le C||\bm{x}-\bm{y}||, \quad \forall \bm{x}, \bm{y} \in \mathcal{R}^n, t \ge 0.
        \end{split}
    \end{equation}
    Then, for every $\bm{x}_0 \in \mathcal{R}^n$, there exists a unique continuous and adapted process $(\bm{x}_{t}^{x_0})_{t\ge 0}$ such that for $t\ge 0$
    \begin{equation}
        \bm{x}_t^{\bm{x}_0} = \bm{x}_0 +\int_0^tf(\bm{x}_s^{\bm{x}_0},t;\bm{\theta}_f)ds + \int_0^t g(\bm{x}_0;\bm{\theta}_g) dW_s
    \end{equation}
    Moreover, for every $T\ge 0$, $\textup {E}(\textup{sup}_{1\le s \le T}|\bm{x}_s|^2)<+\infty.$
    \label{theorem:uniqueness}
    \end{theorem}
    
    The proof of Theorem \ref{theorem:uniqueness} can be found in the supplementary material.

    \textbf{Remark}. According to Theorem \ref{theorem:uniqueness},
    $f(\bm{x},t;\bm{\theta}_f)$ and $g(\bm{x};\bm{\theta}_g)$ must both be
    uniformly Lipschitz continuous.  This can be satisfied by using Lipshitz
    nonlinear activations in the network architectures, such as ReLU, sigmoid and
    Tanh \citep{2018:Anil}.  However, if we na\"{\i}vely optimize the loss function
    in Equation \eqref{eq:loss}, $g(\bm{x}_0;\bm{\theta}_g)$ can be infinitely
    large for the input from out-of-distribution.  This will lead to explosive
    solution and make the optimization procedure unstable.  To solve this problem,
    we define the maximum value of the output of $g(\bm{x};\bm{\theta}_g)$ as a
    hyper-parameter $\sigma_{\rm max}$. Then, the output of the diffusion neural
    net is given by a sigmoid function times $\sigma_{\rm max}$.
    
    \subsection{SDE-Net Training}
    
    There is no closed form solution to the true final random variable $\bm{x}_T$.
    In principle, we can simulate the stochastic dynamics using any high-order
    numerical solver with adaptive step size \citep{platen1999introduction}.
    However, high-order numerical methods can be costly in the context of deep
    learning where the input can have thousands of dimensions.  Since we focus on
    supervised learning and uncertainty quantification, we choose to use the simple
    Euler-Maruyama scheme with fixed step size \citep{Kloeden} for efficient
    network training. Under such a scheme, the time interval $[0,T]$ is divided into
    $N$ subintervals. Then, we can simulate the SDE by: 
    \begin{equation}
         \bm{x}_{k+1} = \bm{x}_k + f(\bm{x}_k, t; \bm{\theta}_f)\Delta t + g(\bm{x}_0;\bm{\theta}_g)\sqrt{\Delta t}Z_k
           \label{eq:discretize}
    \end{equation}
    where $Z_k \sim \mathcal{N}(0,1)$ is the standard Gaussian random variable and
    $\Delta t = T/N$. We will show that empirically it suffices to sample only one
    path for each data point during training time.  The number of steps for solving
    the SDE can be considered equivalently as the number of layers in the
    definition of traditional neural nets.  Then, the training of SDE-Net is
    actually the forward and backward propagations as in standard neural nets,
    which can be easily implemented with libraries such as Tensorflow and Pytorch.
    The drift neural net $f$ and the diffusion neural net $g$ are optimized
    alternately, as shown in Algorithm \ref{alg:training}.

    \begin{algorithm}[t]
        \caption{Training of SDE-Net. $h_1$ is the downsampling layer; $h_2$ is the fully connected layer; $f$ and $g$ are the drift net and diffusion net; $L$ is the loss function.}
        \begin{algorithmic}
        \STATE Initialize $h_1, f, g$ and $h_2$
            \FOR{$\#$ training iterations}
            \STATE Sample minibatch of $N_M$ data from in-distribution: $\mathbf{X}^{N_M} \sim p_{\text {train}}(x)$
    
            \STATE Forward through the downsampling layer: $\mathbf{X}_0^{N_M} = h_1(\mathbf{X}^{N_M})$
    
            \STATE Forward through the SDE-Net block:
                \FOR{$k=0$ to $N-1$}
                \STATE Sample $\mathbf{Z}_k^{N_M} \sim \mathcal{N}(0,\mathbf{I})$
                \STATE $\mathbf{X}_{k+1}^{N_M} = \mathbf{X}_k^{N_M} + f(\mathbf{X}_k^{N_M}, t)\Delta t + g(\mathbf{X}_0^{N_M})\sqrt{\Delta t}\mathbf{Z}_k$
                \ENDFOR 
            \STATE Forward through the fully connected layer: $\mathbf{X}_{\text{f}}^{N_M} = h_2(\mathbf{X}_k^{N_M})$
            \STATE Update $h_1, h_2$ and $f$ by $\nabla_{h_1,h_2,f}\frac{1}{N_M}L(\mathbf{X}_{\text{f}}^{N_M})$
            \STATE Sample minibatch of $N_M$ data from out-of-distribution: $\mathbf{X}^{N_M} \sim p_{\text {OOD}}(x)$
            \STATE Forward through the downsampling layer: $\mathbf{X}_0^{N_M}, \tilde{\mathbf{X}}_0^{N_M} = h_1(\mathbf{X}^{N_M}), h_1(\tilde{\mathbf{X}}^{N_M})$
            \STATE Update $g$ by $\nabla_g g(\mathbf{X}_0^{N_M})-\nabla_g g(\tilde{\mathbf{X}_0}^{N_M})$
            \ENDFOR
        \end{algorithmic}
        \label{alg:training}
    \end{algorithm}

    \section{Experiments} 
    \label{sect:exp}
    
    In this section, we study how the estimated uncertainty can improve model
    robustness and label efficiency. We first study three tasks on model
    robustness: (1) out-of-distribution detection, (2) misclassification detection,
    and (3) adversarial sample detection. We then study how the estimated
    uncertainties can improve label efficiency on active learning.
    
    \subsection{Experimental Setup} 
    
    We compare our SDE-Net model with the following methods: (1) Threshold
    \citep{Hendrycks:2016}, which is used in the deterministic DNNs (2) MC-dropout
    \citep{Gal2015}, (3) DeepEnsemble\footnote{We use five neural nets in the
    ensemble.}\citep{Lakshminarayanan2016}, (4) Prior network (PN)
    \citep{Malinin2018}, (5) Bayes by Backpropagation (BBP) \citep{Blundell:2015},
    (6) preconditioned Stochastic gradient Langevin dynamics (p-SGLD)
    \citep{Li:2016:PSG:3016100.3016149}.  
    
    The network architecture of the compared methods is a residual net
    \citep{Chen2018}. For our method, we use one SDE-Net block in place of
    residual blocks and set the number of subintervals as  the number of residual
    blocks in ResNet for fair comparison---the number of hidden layers in SDE-Net
    is the same as the baseline models under such settings. For our SDE-Net, we
    sample one path during training and perform 10 stochastic forward passes at
    test time in all experiments. 
    
    As PN and SDE-Net both involve OOD samples during the training process, we
    purturb the training data with Gaussian noise (zero mean and variance four for
    both MNIST and SVHN) as pseudo OOD data.  Our supplementary materials provide
    more details about the implementation, setup, and additional experimental
    results.

    \begin{table*}[t]
        \scriptsize
        \caption{Classification and out-of-distribution detection results on MNIST and SVHN. All values are in percentage, and larger values indicates better detection performance. We report the average performance and standard deviation for 5 random initializations. }
          \label{table:OOD}
          \centering
          \begin{tabular}{cccccccccc}
            \toprule
        ID &  OOD & Model &\# Parameters&\makecell{Classification \\accuracy}& \makecell{TNR \\at TPR $95\%$} & AUROC& \makecell{Detection \\accuracy} & \makecell{AUPR \\in} & \makecell{AUPR\\ out}\\
            \midrule
           \multirow{4}{*}{MNIST}& \multirow{4}{*}{SEMEION}
           &Threshold& 0.58M & $99.5 \pm0.0$ &$94.0\pm 1.4$ & $98.3\pm 0.3$ &$94.8\pm 0.7 $ & $99.7\pm 0.1$ & $89.4 \pm 1.1$\\
           &&DeepEnsemble& 0.58M $\times$ 5 &$\bm{99.6}\pm$ NA &$96.0\pm$ NA &$98.8\pm$ NA&$95.8\pm$ NA & $99.8 \pm $ NA& $91.3\pm$ NA   \\
            &&MC-dropout & 0.58M& $99.5 \pm 0.0$ & $92.9\pm 1.6$ & $97.6\pm0.5$ & $94.2\pm0.7$ & $99.6\pm0.1$ & $88.5\pm1.7$ \\
             &&PN&0.58M & $99.3 \pm 0.1$ & $93.4 \pm 2.2$  & $96.1\pm 1.2$ & $94.5\pm 1.1$ & $98.4\pm 0.7$  & $88.5 \pm 1.3$ \\
                 &&BBP& 1.02M& $99.2 \pm 0.3$& $75.0\pm 3.4$ & $94.8\pm 1.2$ & $90.4\pm 2.2$ & $99.2\pm 0.3$ & $76.0 \pm 4.2$ \\
             &&p-SGLD&0.58M&  $99.3 \pm 0.2$ & $85.3 \pm 2.3$ & $89.1 \pm 1.6$  & $90.5\pm 1.3$ & $93.6\pm 1.0$ & $82.8\pm 2.2$ \\
             &&SDE-Net  &0.28M& $99.4\pm 0.1$& $\bm{99.6}\pm 0.2$ & $\bm{99.9}\pm 0.1$ & $\bm{98.6}\pm 0.5$ & $\bm{100.0}\pm 0$ & $\bm{99.5}\pm 0.3$ \\
            \hline
            \multirow{4}{*}{MNIST}&\multirow{4}{*}{SVHN} 
            &Threshold& 0.58M&$99.5 \pm 0.0$ & $90.1 \pm 2.3$ &$96.8 \pm 0.9$&$92.9 \pm 1.1$ & $90.0 \pm 3.5$ &  $98.7 \pm 0.3$\\
            &&DeepEnsemble&0.58M$\times 5$ & $\bm{99.6}\pm $ NA &$92.7\pm$ NA& $98.0 \pm$ NA & $94.1\pm$ NA & $94.5\pm$ NA& $99.1\pm$ NA   \\
                &&MC-dropout&0.58M &$99.5 \pm 0.0$& $88.7\pm 0.6$ & $95.9\pm 0.4$ & $92.0\pm 0.3$ & $87.6\pm 2.0$ & $98.4\pm 0.1$ \\
                &&PN& 0.58 M& $99.3 \pm 0.1$ & $90.4\pm 2.8$ & $94.1 \pm 2.2$ & $93.0\pm 1.4$ & $73.2\pm 7.3$ & $98.0\pm 0.6$ \\
                &&BBP&1.02M & $99.2 \pm 0.3$ & $80.5\pm 3.2$ & $96.0\pm 1.1$ & $91.9\pm 0.9$& $92.6\pm 2.4$ & $98.3 \pm 0.4$ \\
                &&p-SGLD&0.58M & $99.3 \pm 0.2$& $94.5\pm 2.1$ & $95.7\pm 1.3$ & $95.0\pm 1.2$  & $75.6\pm 5.2$ & $98.7 \pm 0.2$ \\
              &&SDE-Net  & 0.28M&$99.4 \pm 0.1$& $\bm{97.8}\pm 1.1$ & $\bm{99.5}\pm 0.2$ & $\bm{97.0}\pm 0.2$ & $\bm{98.6}\pm 0.6$ & $\bm{99.8}\pm 0.1$ \\
              \hline
                \multirow{4}{*}{SVHN}&  \multirow{4}{*}{CIFAR10}
                &Threshold& 0.58M& $95.2 \pm 0.1$& $66.1 \pm 1.9$ & $94.4 \pm 0.4$&$89.8 \pm 0.5$&$96.7 \pm 0.2$ &$84.6 \pm 0.8$ \\
                &&DeepEnsemble& 0.58M$\times$5 & $\bm{95.4}\pm$ NA& $66.5\pm$ NA &$94.6 \pm$ NA & $90.1 \pm$  NA& $97.8 \pm$ NA& $84.8 \pm$ NA    \\
           &&MC-dropout & 0.58M& $95.2 \pm 0.1$& $66.9 \pm 0.6$ &$94.3 \pm 0.1$ & $89.8 \pm 0.2$ & $97.6 \pm 0.1$ & $84.8 \pm 0.2$ \\
              &&PN& 0.58M& $95.0 \pm 0.1$& $66.9\pm 2.0$ & $89.9\pm 0.6$ & $87.4 \pm 0.6$ & $92.5 \pm 0.6$ & $82.3 \pm 0.9$ \\
             &&BBP & 1.02M & $93.3\pm 0.6$ & $42.2\pm 1.2$ & $90.4 \pm 0.3$ & $83.9 \pm 0.4$ & $96.4 \pm 0.2$ & $73.9\pm 0.5$ \\
           &&p-SGLD& 0.58M& $94.1 \pm 0.5$& $63.5\pm 0.9$ & $94.3 \pm 0.4$ & $87.8 \pm 1.2$ & $97.9\pm 0.2$ & $83.9\pm 0.7$ \\
            &&SDE-Net  &0.32 M &$94.2\pm 0.2$ & $\bm{87.5}\pm 2.8$ & $\bm{97.8}\pm 0.4$ & $\bm{92.7}\pm 0.7$ & $\bm{99.2}\pm 0.2$ & $\bm{93.7}\pm 0.9$ \\
            \hline
            \multirow{4}{*}{SVHN}&\multirow{4}{*}{CIFAR100} &Threshold&0.58M & $95.2 \pm 0.1$&$64.6 \pm 1.9$ &$93.8 \pm 0.4$&$88.3 \pm 0.4$& $97.0 \pm 0.2$ & $83.7 \pm 0.8$  \\
            &&DeepEnsemble&0.58M $\times 5$&$\bm{95.4} \pm$ NA &$64.4\pm $ NA & $93.9 \pm $ NA &  $89.4 \pm $ NA& $97.4 \pm$ NA & $84.8 \pm$ NA \\
                &&MC-dropout &0.58M & $95.2\pm 0.1$& $65.5 \pm 1.1$ & $93.7 \pm 0.2$ & $89.3 \pm 0.3$ & $97.1 \pm 0.2$& $83.9 \pm 0.4$  \\
                &&PN&0.58M &$95.0\pm 0.1$ & $65.8\pm 1.7$ & $89.1\pm 0.8$ & $86.6 \pm 0.7$ & $91.8 \pm 0.8$ & $81.6 \pm 1.1$\\
                &&BBP &1.02M & $93.3 \pm 0.6$& $42.4 \pm 0.3$ & $90.6\pm 0.2$ & $84.3\pm 0.3$ & $96.5\pm 0.1$ & $75.2\pm 0.9$ \\
                  &&p-SGLD& 0.58M&$94.1\pm 0.5$ & $62.0\pm 0.5$ & $91.3\pm 1.2$ & $86.0\pm 0.2$ & $93.1\pm 0.8$ & $81.9\pm 1.3$ \\
              &&SDE-Net  & 0.32M&$94.2 \pm 0.2 $& $\bm{83.4}\pm 3.6$& $\bm{97.0}\pm 0.4$ & $\bm{91.6}\pm 0.7$& $\bm{98.8}\pm 0.1$& $\bm{92.3}\pm 1.1$ \\
            \bottomrule
          \end{tabular}
    \end{table*}

    \subsection{Out-of-Distribution Detection}
    \label{sec:OOD}
    
    Our first task is out-of-distribution (OOD) detection, which aims to use
    uncertainty to help the model recognize out-of-distribution samples at test
    time.  In open-world settings, the model needs to deal with continuous data
    that may come from different data distributions or unseen classes. For OOD
    samples, it is wiser to let the model say `I don't know' instead of making an
    absurdly wrong predictions. We investigate the OOD detection task under both
    classification and regression settings. Following previous work
    \citep{Hendrycks:2016}, we use four metrics for the OOD detection task: (1)
    True negative rate (TNR) at $95\%$ true positive rate (TPR); (2) Area under the
    receiver operating characteristic curve (AUROC); (3) Area under the
    precision-recall curve (AUPR); and (4) Detection accuracy. Larger values of
    them indicate better detection performance.

    \textbf{OOD detection for classification.} We first evaluate the performance of
    different models for OOD detection in classification tasks. For fair
    comparison, all the methods use the probability of the final predicted class
    for detection. Table \ref{table:OOD} shows the OOD detection performance as well as the
     classification accuracy on two image classification datasets: MNIST and SVHN.
     We mix different test OOD datasets with the target dataset (MNIST or SVHN) and
     evaluate the performance of different models in OOD detection.  As shown,
     SDE-Net consistently achieves the best OOD detection performance among all the
     models under different combinations.  DeepEnsemble is
     the strongest among the baselines but it still underperforms SDE-Net
     consistently.  Furthermore, DeepEnsemble needs to train multiple DNNs and
     incurs much larger computational costs. 
     While PN and SDE-Net both use pseudo OOD data (with Gaussian
     noise) during training, SDE-Net consistently outperforms PN in all the
     settings.  In addition to using the Gaussian-purturbed OOD data, we also
     compared the performance of SDE-Net and PN when using real-life OOD datasets
     during training (see supplementary material).  We find that PN is easy to be
     overfitted, while our SDE-Net is more robust to the choice of OOD data used
     for training.

    \begin{figure}[h]
      \centering
      \includegraphics[width=0.95\linewidth]{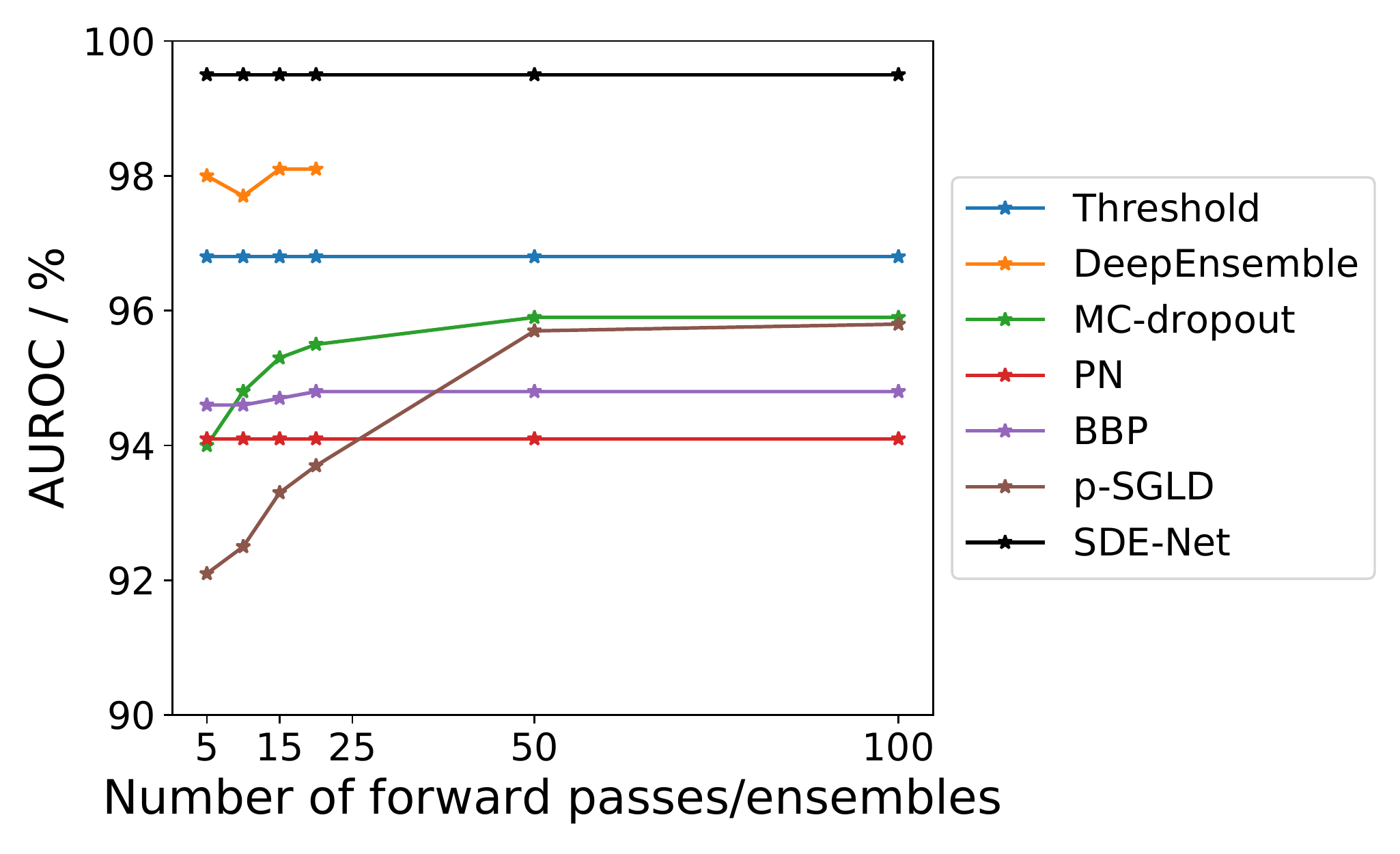}
      \caption{Effect of number of forward passes / ensembles on out-of-distribution (OOD) detection. We use MNIST as the ID data and SVHN as the OOD data.}
      \label{fig:OOD}
    \end{figure}
    
    Fig.~\ref{fig:OOD} shows the impact of the number of forward passes or ensembles on OOD detection, using MNIST as the ID data and SVHN as the OOD data.
    As we can see, the BNNs (MC-dropout, p-SGLD and BBP) require more samples than SDE-Net to reach their peak performance at test time. For DeepEnsemble, its performance is already almost saturated when using five nets and larger ensemble sizes can bring little performance gain.

    In addition to the OOD detection metrics, we also studied the classification accuracy of different models. We find that the predictive
    performance of SDE-Net is very close to state-of-the-art results even with
    significantly fewer parameters. One can further achieve better results by
    stacking multiple SDE-Net blocks together.

    \textbf{OOD detection for regression.} We now investigate OOD detection in
    regression tasks. Different from classification, few works have studied the OOD
    detection task for regression. We use the Year Prediction MSD dataset
    \citep{Dua:2019} as training data and the Boston Housing dataset \citep{Boston}
    as test OOD data. Threshold and PN are excluded here since they only apply to classification
    tasks. To detect OOD samples for regression tasks, all the methods rely on the
    variance of the predictive mean.  Table \ref{table:OOD-regression} shows the
    OOD detection performance for different methods. The results of other metrics are put into the supplementary material due to the space limit.
    Because of the imbalance of the test ID and OOD data, AUPR out is a better metric than AUPR in. 
    OOD detection for regression is more difficult than for classification, because
    regression is a continuous and unbounded problem which makes uncertainty
    estimation difficult. For this challenging task, all the baselines
    perform quite poorly, yet SDE-Net still achieves strong performance.  The
    reason is that the diffusion net in SDE-Net directly models the relationship
    between the input data and epistemic uncertainty, which encourages SDE-Net to
    output large uncertainty for OOD data and low uncertainty for ID data even for
    this challenging task.
    
    \begin{table}[h]
      \scriptsize
      \caption{Out-of-distribution detection for regression on Year Prediction MSD + Boston Housing. We report the average performance and standard deviation for 5 random initializations.}
          \label{table:OOD-regression}
          \centering
          \begin{tabular}{ccccc}
          \toprule
          Model & \# Parameters &RMSE & AUROC&  \makecell{AUPR\\ out}\\
          \midrule
          DeepEnsemble&14.9K$\times 5$ &$\bm{8.6} \pm$ NA& $59.8\pm$ NA& $1.3\pm$ NA    \\
          MC-dropout &14.9K & $8.7 \pm 0.0$  & $53.0\pm 1.2$  & $1.1\pm 0.1$ \\
          BBP &30.0K &$9.5 \pm 0.2$  & $56.8\pm 0.9$  &$1.3\pm 0.1$  \\
          p-SGLD &14.9K & $9.3 \pm 0.1$& $52.3\pm 0.7$ & $1.1\pm 0.2$ \\
          SDE-Net & 12.4K& $8.7 \pm 0.1$& $\bm{84.4}\pm 1.0$  & $\bm{21.3}\pm 4.1$ \\
          \bottomrule
          \end{tabular}
      \end{table}

    \subsection{Misclassification Detection} Besides OOD data detection, another
    important use of uncertainty is to make the model aware when it may make
    mistakes at test time. Thus, our second task is misclassification detection \citep{Hendrycks:2016},
    which aims at leveraging the predictive uncertainty to identify test samples on
    which the model have misclassified. Table \ref{table:mis} shows the
    misclassification detection results for different models on MNIST and SVHN.
    p-SGLD achieves the best overall performance for this task. 
    SDE-Net achieves comparable performance with DeepEnsemble and outperforms other baselines. 
    However, p-SGLD needs to store the copies of the
    parameters for evaluation, which can be prohibitively costly for large NNs.
    DeepEnsemble requires training multiple models and incurs high computational
    cost. Therefore, we argue that SDE-Net is a better choice for the
    misclassification task in practice.
    
    \begin{table}[t]
      \scriptsize
      \caption{Misclassification detection performance on MNIST and SVHN. We report the average performance and standard deviation for 5 random initializations.}
        \label{table:mis}
        \centering
        \begin{tabular}{cccccc}
          \toprule
       Data & Model & AUROC & \makecell{AUPR \\succ} & \makecell{AUPR\\ err}\\
          \midrule
       \multirow{4}{*}{MNIST}
       &Threshold  & $94.3 \pm 0.9$ &$ 99.8 \pm 0.1$ &$31.9 \pm 8.3$ & \\
       &  DeepEnsemble&$\bm{97.5} \pm$ NA&$\bm{100.0}\pm$ NA&$41.4\pm$ NA    \\
              &MC-dropout&$95.8 \pm 1.3$&$99.9 \pm 0.0$&$33.0\pm 6.7$  \\
          &PN & $91.8\pm 0.7 $& $99.8\pm0.0$ & $33.4 \pm 4.6$\\
              &BBP  & $96.5 \pm 2.1$  & $\bm{100.0}\pm 0.0$ & $35.4\pm 3.2$\\
          &P-SGLD  & $96.4\pm1.7$ & $\bm{100.0\pm}0.0$ & $\bm{42.0\pm}2.4$\\
          &SDE-Net&$96.8\pm 0.9$ &$\bm{100.0 \pm} 0.0$&$36.6 \pm 4.6$   \\
          \hline
           \multirow{4}{*}{SVHN} 
           &Threshold  &$90.1\pm 0.3$ &$99.3\pm 0.0$ & $42.8\pm 0.6$& \\
           &  DeepEnsemble&$91.0\pm$ NA &$\bm{99.4}\pm$ NA &$46.5\pm$ NA    \\
              &MC-dropout&$90.4\pm 0.6$ &$99.3\pm 0.0$ &$45.0\pm 1.2$  \\
          &PN & $84.0\pm 0.4$ & $98.2\pm0.2$ & $43.9\pm 1.1$\\
          &BBP  & $91.8 \pm 0.2$  & $99.1\pm 0.1$ & $50.7\pm 0.9 $\\
          &P-SGLD  & $\bm{93.0}\pm 0.4$ & $\bm{99.4}\pm 0.1$ & $48.6 \pm 1.8$\\
          &SDE-Net&$92.3\pm 0.5$ &$\bm{99.4\pm}0.0$&$\bm{53.9}\pm 2.5$  \\
          \bottomrule
        \end{tabular}
    \end{table}
    
    \subsection{Adversarial Sample Detection}

    Our third task studies adversarial sample detection.  Existing works \citep{Intri,Good:2014} have shown that DNNs
    are extremely vulnerable to adversarial examples crafted by adding small
    adversarial perturbations. The ability to detect such adversarial samples is
    important for AI safety.  Different from existing literature on adversarial
    training, we do not use adversarial training but only examine the
    uncertainty-aware models' ability in detecting adversarial samples.  We study
    two attacks: Fast Gradient-Sign Method (FGSM) \citep{FGSM}, and Projected
    Gradient Descent (PGD) \citep{madry2018towards}.

    Fig.~\ref{fig:adv:FGSM} shows the detection performance of different models when facing
    FGSM attacks.  As shown, when the perturbation size $\epsilon$ varies, SDE-Net
    can achieve similar AUROC with p-SGLD and outperforms all other methods.  On
    the simpler MNIST dataset, all methods can achieve $\sim$100\% AUROC when the
    perturbation size is large. However, on the more challenging SVHN dataset, only
    SDE-Net still converges to 100\% AUROC, while other baselines achieve only
    about 90\% AUROC even with perturbation size of one.

    Fig.~\ref{fig:adv:PGD} shows the detection performance of different models when
    facing PGD attacks.  We use the default parameters in \citep{madry2018towards}
    and plot the AUROC curve versus the number of PGD iterations.  Under the
    stronger PGD attacks, the AUROCs of all the baselines on MNIST drop below 70\%
    after 60 iterations, while SDE-Net can still achieves over 80\% AUROC after 100
    iterations.  On SVHN, we observe a different picture where all the methods
    quickly become overconfident except for the costly DeepEnsemble method.  This
    is likely due to higher dimensionality of the data manifold in SVHN.  Further
    work is needed to design efficient and robust uncertainty-aware models that can
    detect high-dimensional adversarial samples generated by such strong attackers.
    
    \begin{figure}[t!]
      \centering
      \begin{subfigure}[b]{0.48\linewidth}
        \includegraphics[width=\linewidth]{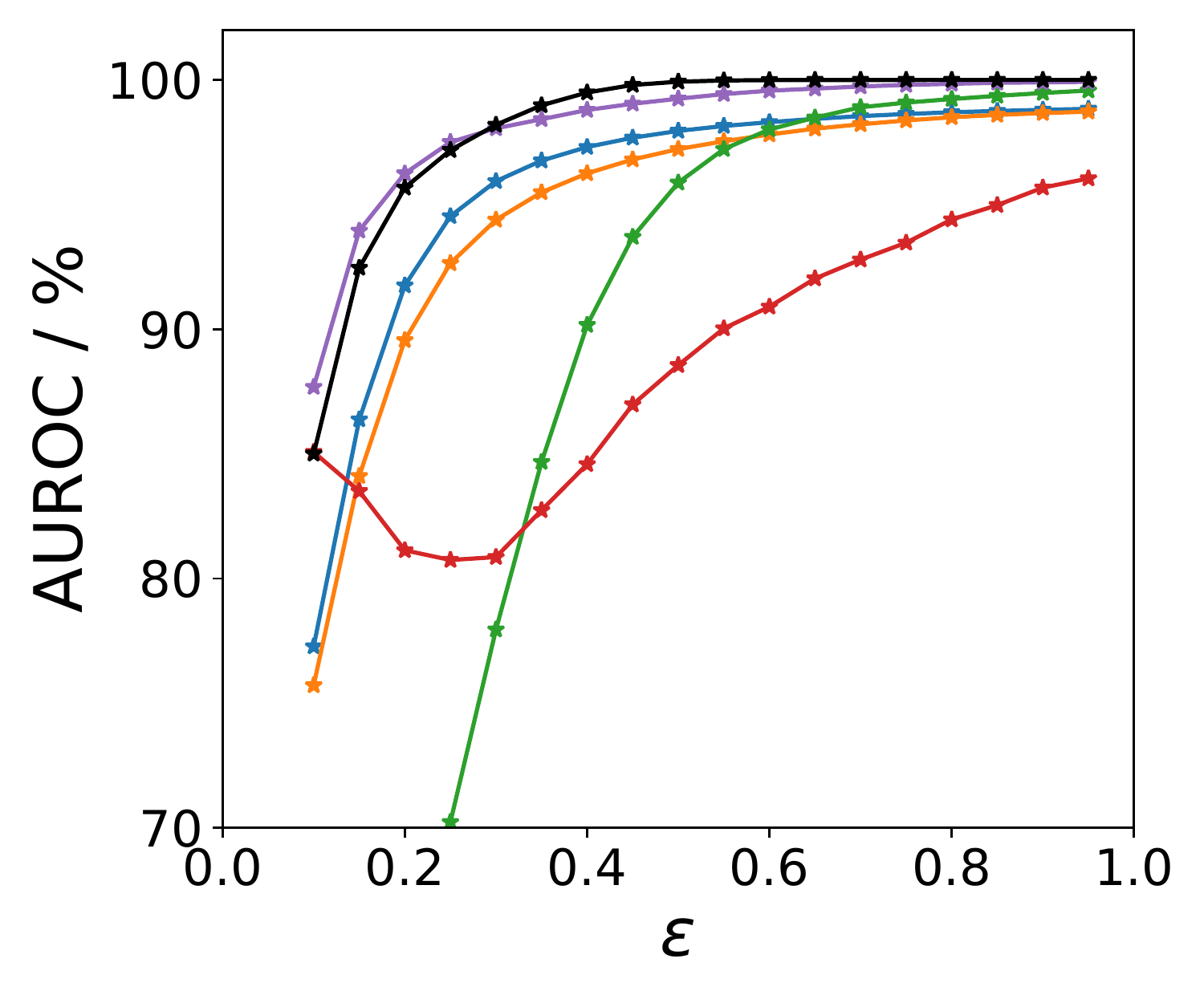}
        \caption{MNIST}
      \end{subfigure}
      \begin{subfigure}[b]{0.48\linewidth}
        \includegraphics[width=\linewidth]{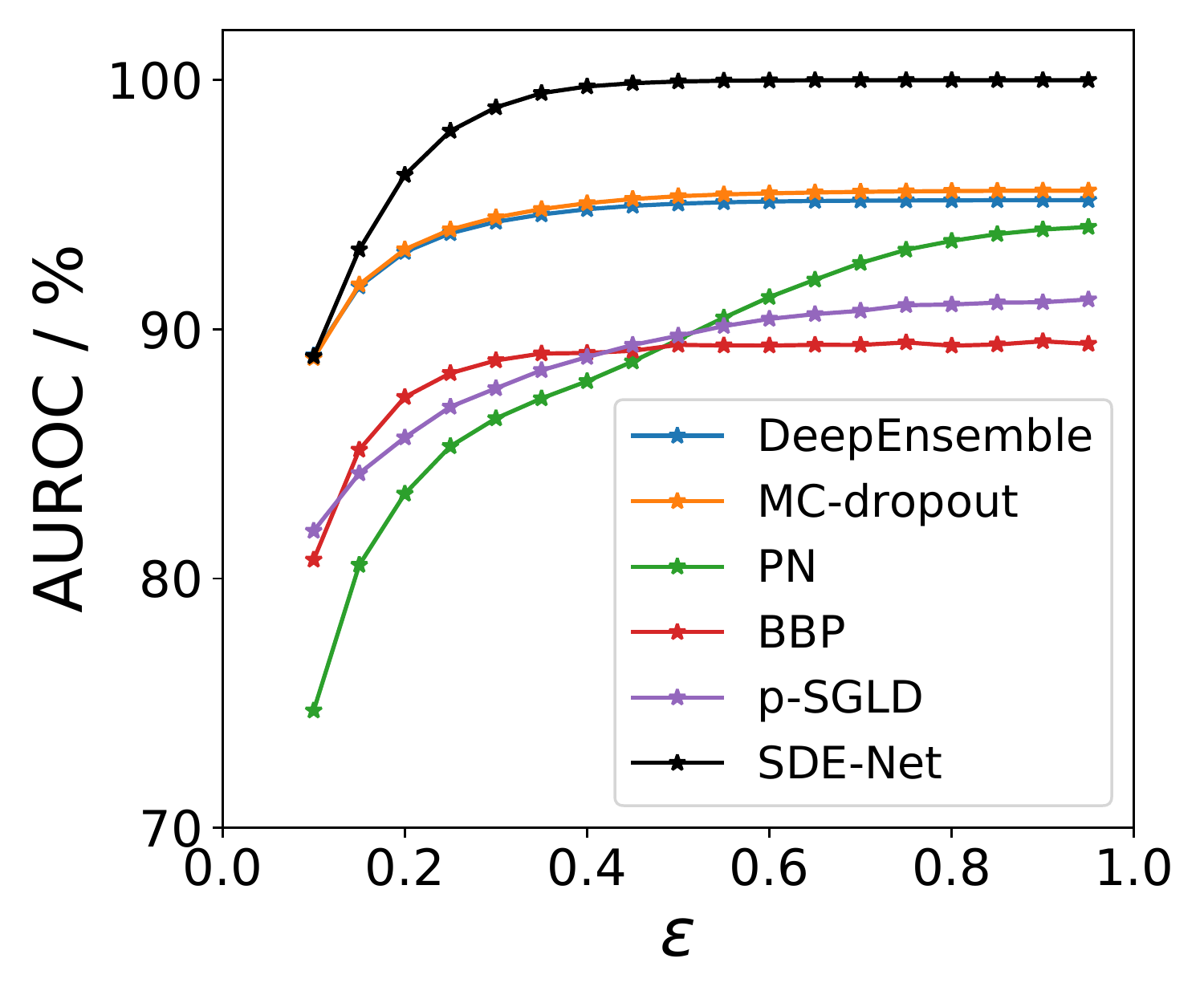}
        \caption{SVHN}
      \end{subfigure}
      \caption{The performance of adversarial sample detection under FGSM attacks. $\epsilon$ is the step size in FGSM.}
      \label{fig:adv:FGSM}
    \end{figure}
    
    \begin{figure}[t!]
      \centering
      \begin{subfigure}[b]{0.48\linewidth}
        \includegraphics[width=\linewidth]{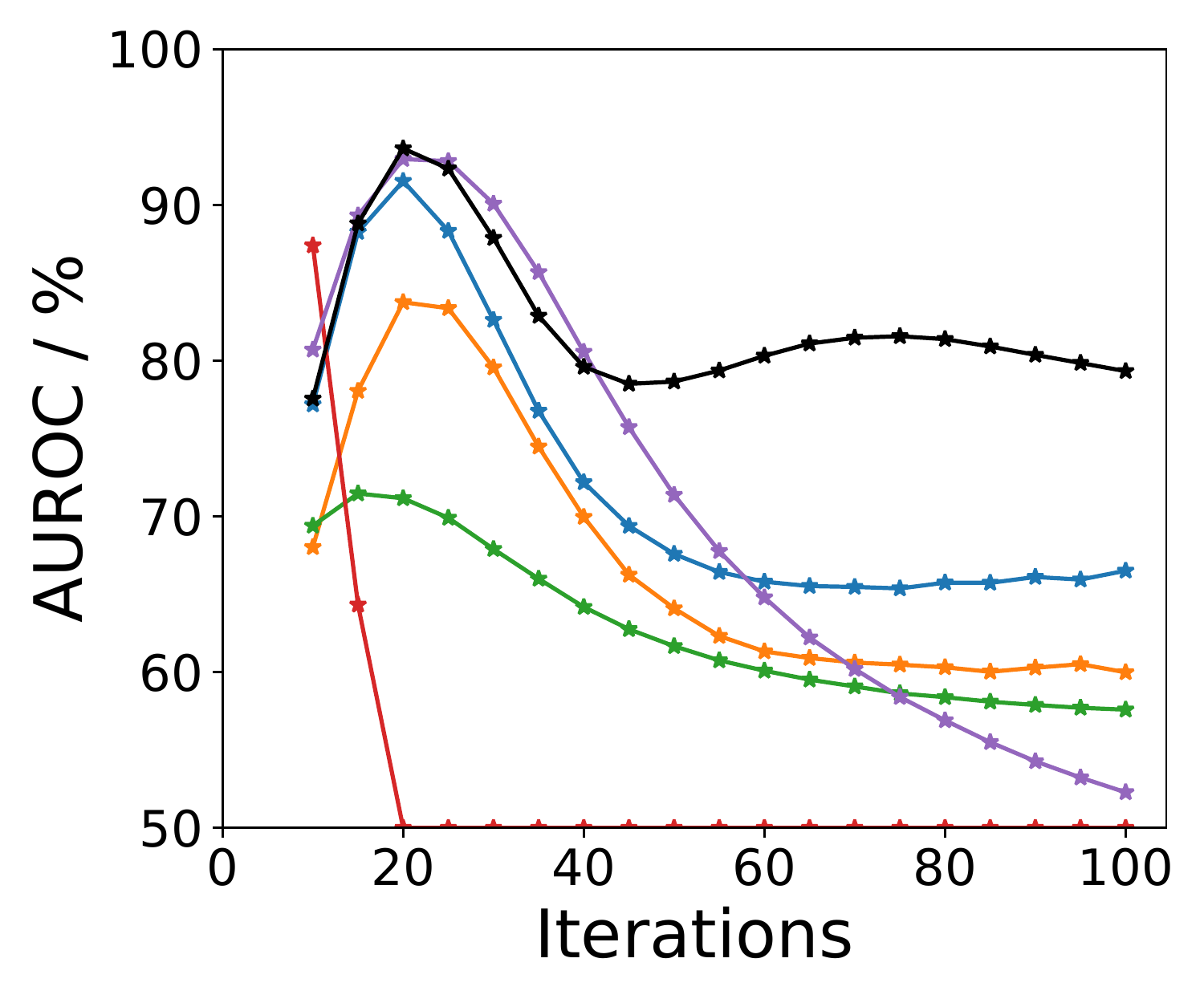}
        \caption{MNIST}
      \end{subfigure}
      \begin{subfigure}[b]{0.48\linewidth}
        \includegraphics[width=\linewidth]{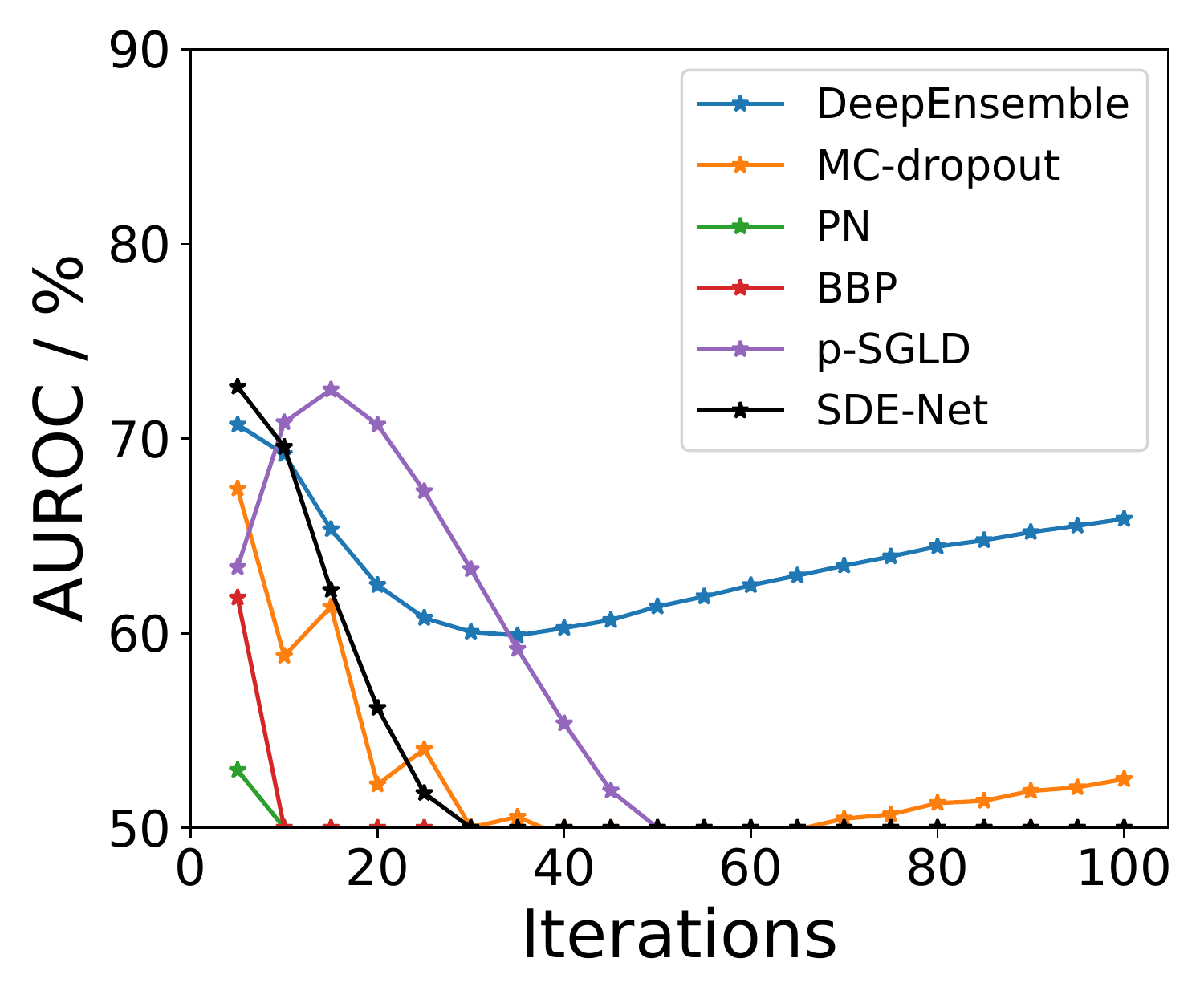}
        \caption{SVHN}
      \end{subfigure}
      \caption{The performance of adversarial sample detection under PGD attacks.}
      \label{fig:adv:PGD}
    \end{figure}

    \subsection{Active Learning}
    
    Finally, we study how the estimated uncertainties can improve label
    efficiency for active learning. Uncertainty plays an important role in
    active learning. Intuitively, accurate uncertainty estimates can dramatically
    reduce the amount of labeled data for model training, while inaccurate
    estimates make the model choose uninformative instances and even lead to worse
    performance due to overfitting. For active learning, we use the acquisition
    function proposed in \citep{hafner2018reliable}:
    \begin{equation}
        \{\bm{x}_{\rm new}, y_{\rm new}\} \sim p_{\rm new}(\bm{x},y) \propto \left( 1+\frac{{\rm Var}[\mu(\bm{x})]}{\sigma^2(\bm{x})} \right)^{2}.
    \end{equation}
    This acquisition function allows us to extract the data from the region where the model has high
    epistemic uncertainty but the data has low aleatoric noise. For the deterministic neural network, we use the predictive variance as a proxy since it cannot model epistemic uncertainty.
    
    \begin{figure}[t]
        \centering
        \includegraphics[width=0.92\linewidth]{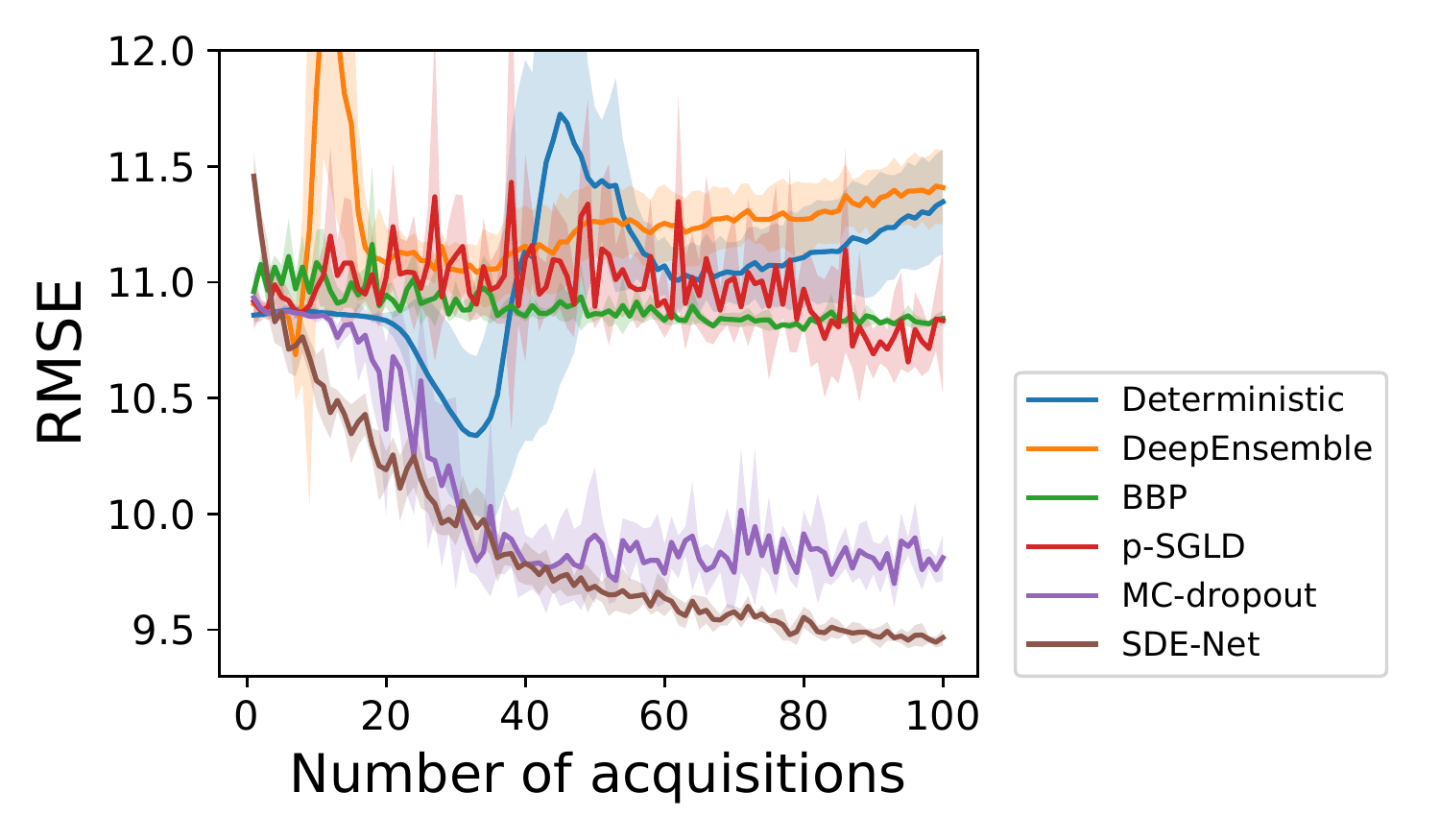}
        \caption{The performance of different models for active learning on the Year Prediction MSD dataset. We report the average performance and standard deviation for 5 random initializations.}
        \label{fig:active_learning}
    \end{figure}
    
    We use the Year Prediction MSD regression dataset, where the task is to predict the release year of a song from 90 audio features. It has total 515,345 data points of which 463,715 are for training.
    We experiment with the following procedure.  Starting from 50 labels, the models select a batch of 50 additional labels in every 100 epochs.  
    The remaining data points in the training dataset are available for acquisition, and we evaluate performance on the whole test set. 
    
    As we can see from Fig.~\ref{fig:active_learning}, 
    the RMSE of SDE-Net consistently decreases as we acquire more labeled data. Such results show that SDE-Net successfully acquire data from informative region. 
    However, the performance gain of BBP and p-SGLD are still negligible even after 100 acquisitions. 
    We can also observe that the performance of the deterministic NN and
    DeepEnsemble start to degrade after several iterations. This is because they
    keep extracting uninformative data points and thus suffer from overfitting due 
    to the small training data size.

    \vspace{-0.5em}
    \section{Additional Related Work}
    
    \textbf{Uncertainty estimation}: BNNs is a principled way for uncertainty quantification. Performing exact Bayesian inference is inefficient and computationally intractable. A common workaround is to use approximation methods like variational inference \citep{Blundell:2015, louizos2017multiplicative, shi2018kernel,louizos2016structured,pmlr-v80-zhang18l}, Laplace approximation \citep{ritter2018a}, expectation propagation \citep{Li:2015}, stochastic gradient MCMC \citep{Li:2016:PSG:3016100.3016149, Welling:2011} and so on. \citet{Gal2015, gal2015bayesian} proposed to use Monte-Carlo Dropout (MC-dropout) at test time to estimate the uncertainty which has a nice interpretation in terms of variational Bayes.
    Another key element which can affect the performance of BNNs is the choice of prior distribution. The most common prior to use is the independent Gaussian distribution which can only give limited and even biased information for uncertainty. Recently, \citet{hafner2018reliable} proposed to use noise contrastive priors (NCPs) to obtain reliable uncertainty estimates. 
    Functional variational BNNs (fBNNs) \citep{sun2018functional} employ Gaussian Process (GP) priors and use BNNs for inference.

    A number of non-Bayesian methods have also been proposed for uncertainty
    quantification. DeepEnsemble \citep{Lakshminarayanan2016} trains an ensemble of
    NNs and reports competitive uncertainty estimates to MC dropout.
    \citet{2017:Pereyra} adds an entropy penalty as the network regularizer. In
    \citep{lee2018training}, the authors proposed to minimize a new confidence loss to both
    a sharp predictive distribution for training data and a flat predictive
    distribution for OOD data. The OOD data is generated by using a generative
    model. Prior network \citep{Malinin2018,malinin2019reverse} parametrized a Dirichlet distribution
    over categorical output distributions which allows high uncertainty for OOD
    data, but it is only applicable to classification tasks.

    \textbf{Neural dynamic system}:  \citet{2017:E} first observed the link between
    ResNet and ODE \citep{Ince1956}. The residual block which is formulated as
    $x_{n+1}=x_n+f(x_n)$ can be considered as the forward Euler discretization of
    the ODE $dx_t=f(x_t)$. In \citep{Lu2017}, the authors show that many
    state-of-the-art deep network architectures, such as PolyNet
    \citep{2016:Zhang}, FractalNet \citep{2016:Larsson} and RevNet
    \citep{2017:Gomez}, can be regarded as different discretization schemes of
    ODEs.  \citet{Chen2018} further generalized the discrete ResNet to a
    continuous-depth network by making use of the existing ODE solvers. The adjoint
    method \citep{2006:Plessix} is used during ODE-Net training, which allows
    constant memory cost and adaptive computation. 
    However, these works all focus on improving the predictive accuracy while our work quantifies model
    uncertainty based on the SDE formulation and the introduced Brownian motion
    term. Concurrently with this paper, \citet{Tzen:2019} establish a connection
    between infinitely deep residual networks and solutions to SDE.
    \citet{li2020scalable} propose a generalization of the adjoint method to
    compute gradients through solutions of SDEs and apply a latent SDE for
    continuous time-series data modeling. Our approaches were developed
    simultaneously but focus on using neural SDEs for uncertainty quantification.

    \vspace{-1.2ex}
    \section{Conclusion} 
    
    We proposed a neural stochastic differential equation model (SDE-Net) for
    quantifying uncertainties in deep neural nets. The proposed model can separate
    different sources of uncertainties compared with existing non-Bayesian methods
    while being much simpler and more straightforward than Bayesian neural nets. 
    Through comprehensive experiments, we demonstrated that SDE-Net has strong performance compared to
    state-of-the-art techniques for uncertainty quantification on both
    classification and regression tasks. To the best of our knowledge, our work
    represents the first study which establishes the connection between stochastic
    dynamical system and neural nets for uncertainty quantification. As the approach
    is general and efficient, we believe this is a promising direction for
    equipping neural nets with meaningful uncertainties in many safety-critical
    applications.

\section*{Acknowledgements}
We would like to thank Srijan Kumar and the anonymous reviewers for their helpful comments.
This work was in part supported by the National Science Foundation award IIS-1418511, CCF-1533768 and IIS-1838042, the National Institute of Health award 1R01MD011682-01 and R56HL138415.

\bibliography{references.bib}
\bibliographystyle{icml2020}

\newpage

\icmltitlerunning{Supplementary Material for SDE-Net: Equipping Deep Neural Networks with Uncertainty Estimates}

\onecolumn
\icmltitle{Supplementary Material for SDE-Net: Equipping Deep Neural Networks with Uncertainty Estimates}

\renewcommand{\thesection}{S. \arabic{section}}

\setcounter{section}{0}
\section{Proof of Theorem 1}
Theorem 1 can be seen as a special case of the existence and uniqueness theorem of a general stochastic differential equation. The following derivation is adapted from \citep{lalley2016stochastic}.
To prove Theorem 1, we first introduce two lemmas.

\setcounter{theorem}{0}
\begin{lemma}
    Let $y(t)$ be a nonnegative function that satisfies the following condition: for some $T\le \infty$, there exist constants $A, B \ge 0$ such that:
    \begin{equation} 
        y(t) \le A+ B\int_0^ty(s)ds <\infty \quad \text{for all} \quad 0\le t \le T.
        \label{eq:lemma1}
    \end{equation}
    Then
    \begin{equation}
        y(t) \le Ae^{Bt} \quad \text{for all} \quad 0\le t\le T.
        \label{eq:lemma12}
    \end{equation}
\end{lemma}

\begin{proof}
    W.l.o.g., we assume that $C=\int_{0}^Ty(s)ds<\infty$ and that $T<\infty$. Then, we can obtain that $y(t)$ is bounded by 
    $D\equiv A+BC$ in the interval $[0,T]$. By iterating over inequality \eqref{eq:lemma1}, we have:
    \begin{equation}
        \begin{split}
            y(t) &\le  A + B\int_{0}^t y(s)ds \\
            &\le A + B\int_0^t(A+B) \int_0^s y(r)drds\\
            & \le A + BAt + B^2\int_0^t\int_0^s (A + B \int_0^r y(q)dq)drds\\
            & \le A + BAt + B^2At^2/2! + B^3\int_0^t\int_0^s\int_0^r(A+B\int_0^qy(p)dp)dqdrds \\
            & \le \cdots.
        \end{split}
    \end{equation}

After $k$ iterations, the first $k$ terms are the series for $Ae^{Bt}$. The last term is a $(k+1)$-fold iterated integral $I_k$. Because
 $y(t) \le D$ in the interval $[0,T]$, the integral $I_k$ is bounded by $B^kDt^{k+1}/(k+1)!$. This converges to zero uniformly for 
$t\le T$ as $k\rightarrow \infty$. Hence, inequality \eqref{eq:lemma12} follows.
\end{proof}

\begin{lemma}
    Let $y_n(t)$ be a sequence of nonnegative functions such that for some constants $B, C < \infty$,
    \begin{equation}
        \begin{split}
            &\quad y_0(t)\le C \quad \text{for all} \quad t\le T \quad \text{and} \\
            &y_{n+1}(t) \le B\int_{0}^ty_n(s)ds<\infty \quad \text{for all} \quad t\le T \quad \text{and} \quad n=0,1,2,\cdots.
        \end{split}
    \end{equation}
    Then,
    \begin{equation}
        y_n(t) \le CB^nt^n/n! \quad \text{for all} \quad t\le T.
    \end{equation}
\end{lemma}

\begin{proof}
    \begin{equation}
        \begin{split}
        y_1 (t) &\le B\int_0^t Cds = BCt \\
        y_2(t) & \le B \int_0^t BCs ds = CB^2t^2/2!\\
        y_3(t) & \le B \int_0^t CB^2s^2/2 ds = CB^3t^3/3! \\
        \cdots.
        \end{split}
    \end{equation}
After $n$ iterations, we have $y_n(t) \le CB^nt^n/n!$ for all $t \le T$.
\end{proof}

Suppose that for some initial value $\bm{x}_0$ there are two different solutions:
\begin{equation}
    \begin{split}
        &\bm{x}_t = \bm{x}_0+\int_{0}^tf(\bm{x}_s, s;\bm{\theta}_f)ds + \int_0^tg(\bm{x}_0;\bm{\theta}_g)dW_s \quad \text{and} \\
        & \bm{y}_t = \bm{x}_0+\int_{0}^tf(\bm{y}_s, s ; \bm{\theta}_f)ds + \int_0^tg(\bm{x}_0 ; \bm{\theta}_g)dW_s.
    \end{split}
\end{equation}
Since the diffusion net $g$ is uniformly Lipschitz, $\int_0^t g(\bm{x_0}; \bm{\theta}_g)dW_s$ is bounded in compact time intervals. Then, we substract these two solutions and get:
\begin{equation}
    \bm{x}_t - \bm{y}_t = \int_0^t (f(\bm{x}_s,s;\bm{\theta}_f)-f(\bm{y}_s,s;\bm{\theta}_f))ds.
\end{equation}
Since the drift net $f$ is uniformly Lipschitz, we have that for some constant $B < \infty$,
\begin{equation}
    |\bm{y}_t -  \bm{x}_t| \le B\int_0^t |\bm{y}_s-\bm{x}_s|ds \quad\text{for all}\quad t<\infty.
\end{equation}
It is obvious that $\bm{y}_t - \bm{x}_t \equiv 0$ from Lemma 1 by letting $A=0$. Thus, the stochastic differential equation has at most one solution for any particular
initial value $\bm{x}_0$. 

Then, we prove the the existence of the solutions. For a fix initial value $\bm{x}_0$, we define a sequence of adapted process $\bm{x}_n(t)$ by:
\begin{equation}
    \bm{x}_{n+1}(t) = \bm{x}_0 + \int_0^t f(\bm{x}_n(s),s;\bm{\theta}_f)ds + g(\bm{x_0}; \bm{\theta}_g)W_t
\end{equation}
The processes $\bm{x}_{n+1}(t)$ are well-defined and have continuous paths, by induction on $n$. Because the drift net $f$ is Lipschitz, we have:
\begin{equation}
    |\bm{x}_{n+1}(t) - \bm{x}_n (t)| \le B \int_0^t |\bm{x}_n(s)-\bm{x}_{n-1}(s)|ds.
\end{equation}
Therefore, Lemma 2 implies that for any $T< \infty$,
\begin{equation}
    |\bm{x}_{n+1}(t)-\bm{x}_n(t)|\le CB^nT^n/n! \quad \text{for all} \quad t \le T.
\end{equation}
It follows that the processes $\bm{x}_n(t)$ converge uniformly in compact time intervals $[0, T]$; thus the limit process $\bm{x}(t)$ has continuous trajectories according to 
the dominated convergence theorem and the continuity of $f$.

\section{Experimental Details}
\label{section:exp}

\subsection{Classification Setup Details}

\textbf{Data preprocessing.} As PN and SDE-Net both require OOD samples during the training process, we perturb training data by Gaussian noise as pseudo OOD data by default.  On both MNIST and SVHN, the mean of the Gaussian noise is set to zero and the variance is set to 4.

We have also experimented with using external data as OOD data for model training or test, which requires re-scaling external data to match the target dataset. Specifically, for the classification task on MNIST, we used SEMEION and upscaled the images to $28\times28$; we also tried CIFAR10 and transformed images into greyscale and downsampled them to $28\times 28$ size. 

\textbf{Model hyperparameters.} we use one SDE-Net block in replace of 6 residual blocks and set the number of subintervals as $N = 6$ for fair comparison.
We perform one forward propagation during training time and 10 forward propagations at test time. 
$\sigma_{\text{max}}=500$ was used for both MNIST and SVHN. To make the training procedure more stable, we use a smaller value of $\sigma_{\text{max}}$ during training. Specifically, we set $\sigma_{\text{max}}=20$ for MNIST and $\sigma_{\text{max}}=5$ for SVHN during trainining.

The dropout rate for MC-dropout is set to 0.1 as in \citep{Lakshminarayanan2016} (we also tested 0.5, but that setting performed worse). For DeepEnsemble, we use 5 ResNets in the ensemble. For PN, we set the concentration parameter to 1000 for both MNIST and SVHN as suggested in the original paper. 
We use the standard normal prior for both BBP and p-SGLD. The variances of the prior are set to 0.1 for BBP and 0.01 for p-SGLD to ensure convergence. We use 50 posterior samples for MC-dropout, BBP and p-SGLD at test time.

For PGD attack, we set the perturbations size $\epsilon$ to $0.3 \; (16/255)$ and step size to $2/255 \; (0.4/255)$ on MNIST (SVHN).

\textbf{Model optimization.} 
On the MNIST dataset, we use the stochastic gradient descent algorithm with momentum $0.9$, weight decay $5\times10^{-4}$, and mini-batch size 128.
BBP and p-SGLD are trained with 200 epochs to ensure convergence while other methods are trained with 40 epochs.
The initial learning rate is set to 0.1 for for drift network, MC-dropout and DeepEnsemble while 0.01 for PN. It then decreased at epoch 10, 20 and 30.
The learning rate for drift network is initially set to 0.01 and then decreased at epoch 15 and 30. The learning rate for BBP is initially set to 0.001 and then decreased at epoch 80 and 160. 
We use an initial learning rate 0.0001 for p-SGLD and then decreased it at epoch 50. The decrease rate for SGD learning rate is set to $0.1$.

On the SVHN dataset, we again use the stochastic gradient descent algorithm with momentum $0.9$ and weight decay $5\times10^{-4}$. 
BBP and p-SGLD are trained with 200 epochs to ensure convergence while other methods are trained with 60 epochs. The initial learning rate is set to 0.1 for 
for drift network, MC-dropout and DeepEnsemble while 0.01 for PN. It then decreased at epoch 20 and 40. The learning rate for diffusion network is set as 0.005 initially and then decreased at epoch 10 and 30. 
 p-SGLD uses a contant learning rate 0.0001. The learning rate for BBP is initially set to 0.001 and then decreased at epoch 80 and 160.

\subsection{Regression Setup Details}

\textbf{Data preprocessing.} We normalize both the features and targets (0 mean and 1 variance) for the regression task. We repeat the features of Boston Housing data 6 times and pad zeroes for the remaining entries to make the number of features of the two datasets equal. We perturb training data by Gaussian noise (zero mean and variance 4) as pseudo OOD data. 

\textbf{Model hyperparameters.}
The neural net used in the baselines has 6-hidden layers with ReLU nonlinearity. For fair comparison, we set the number of subintervals as $4$ and
then place two layers before and after the SDE-Net block respectively. The dropout rate for MC-dropout is set to 0.05 as in \citep{Gal2015}.
We set $\sigma_{\text{max}}$  to 0.01 initially and increase it to 0.5 at epoch 30. During training, we only perform 1 forward pass.
The number of stochastic forward passes is 10 for SDE-Net at test time. 20 posterior samples are used for MC-dropout, BBP and p-SGLD at test time. The variance is set to 0.1 for both BBP and p-SGLD to ensure convergence.

\textbf{Model optimization.}
We use the stochastic gradient descent algorithm with momentum $0.9$, weight decay $5\times10^{-4}$, and mini-batch size 128. The number of training epochs is 60.  
The learning rate for drift net is initially set to $0.0001$ and then deceased at epoch $20$. The learning rate for the diffusion net is set to 0.01. 
The learning rate for BBP and p-SGLD is initially set to $0.01$ and then decreased at epoch $20$. 
The learning rate for other baselines is initially set to $0.001$ and then decreased at epoch $20$.

\subsection{Active Learning Setup}
\textbf{Data preprocessing.} We normalize both the features and targets (0 mean and 1 variance) for the active learning task. We 
randomly select 50 samples from the original training set as the starting point.

\textbf{Model hyperparameters.} The network architecture and model hyperparameters are the same as those we used in the OOD detection task for regression.

\textbf{Model optimization.} We use the stochastic gradient descent algorithm with momentum $0.9$, weight decay $5\times10^{-4}$, and mini-batch size 50. The number of training epochs is 100. The learning rate for drift net and baselines is set to $0.0001$. The learning rate for the diffusion net is set to 0.01.

\section{Additional Experiments}
\subsection{Visulization Using Synthetic Dataset}
In this subsection, we demonstrate the capability of SDE-Net of obtaining meaningful epistemic uncertainties.
For this purpose, we generate a synthetic dataset from a mixture of two Gaussians.
Then, we train the SDE-Net on this toy dataset. Both the drift neural network and diffusion network have
one hidden layer with ReLU activation . 

Figure \ref{fig:un} shows the uncertainty obtained by SDE-Net. Specifically, it visualizes the epistemic uncertainty given by the 
variance of the Brownian motion term. As we can see, the uncertainty is low in the region covered by the training data while high outside the training distribution.

\begin{figure}[h!]
  \centering
  \begin{subfigure}[b]{0.45\linewidth}
   \includegraphics[width=\linewidth]{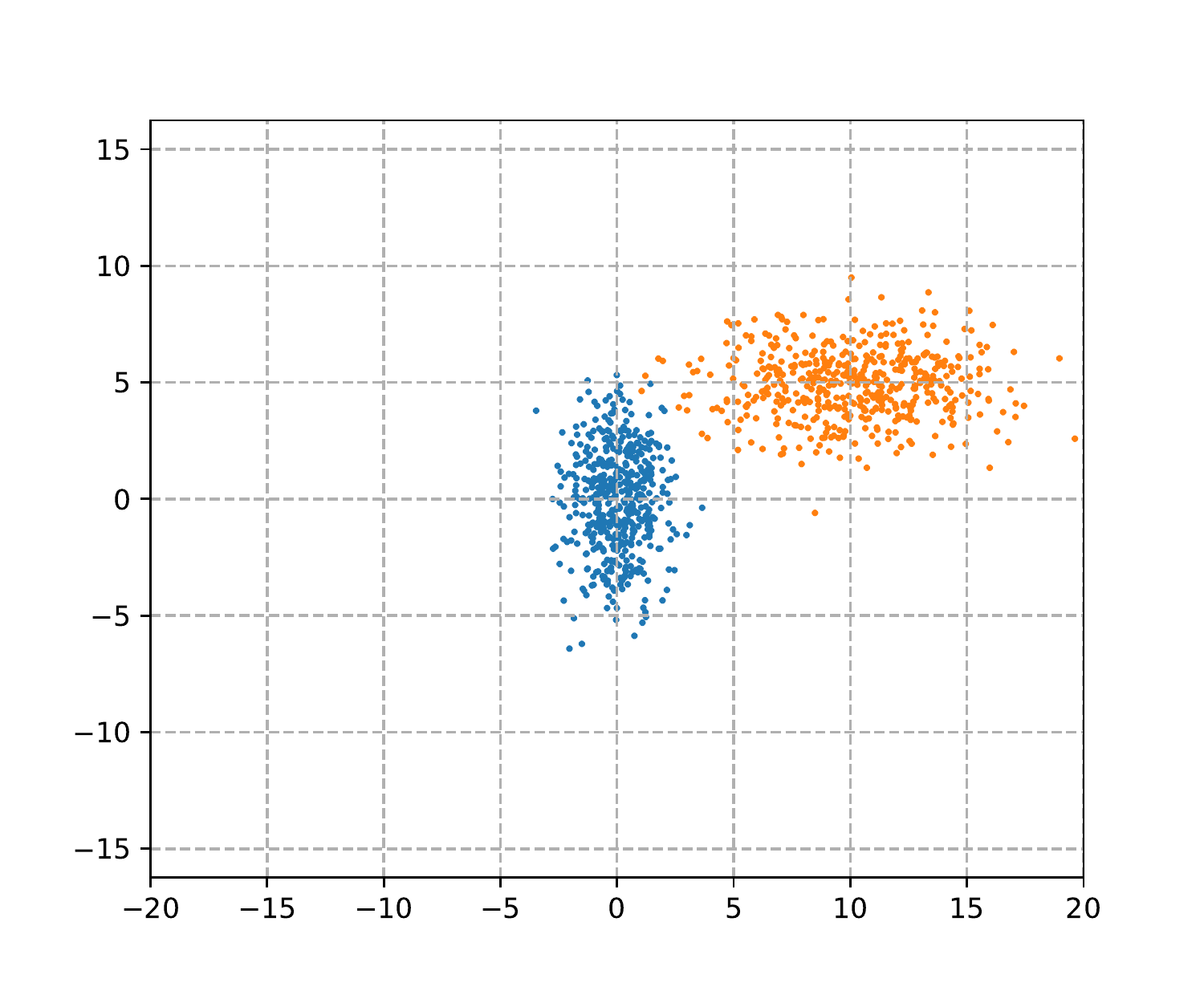}
    \caption{Training data distribution.}
  \end{subfigure}
  \begin{subfigure}[b]{0.45\linewidth}
    \includegraphics[width=\linewidth]{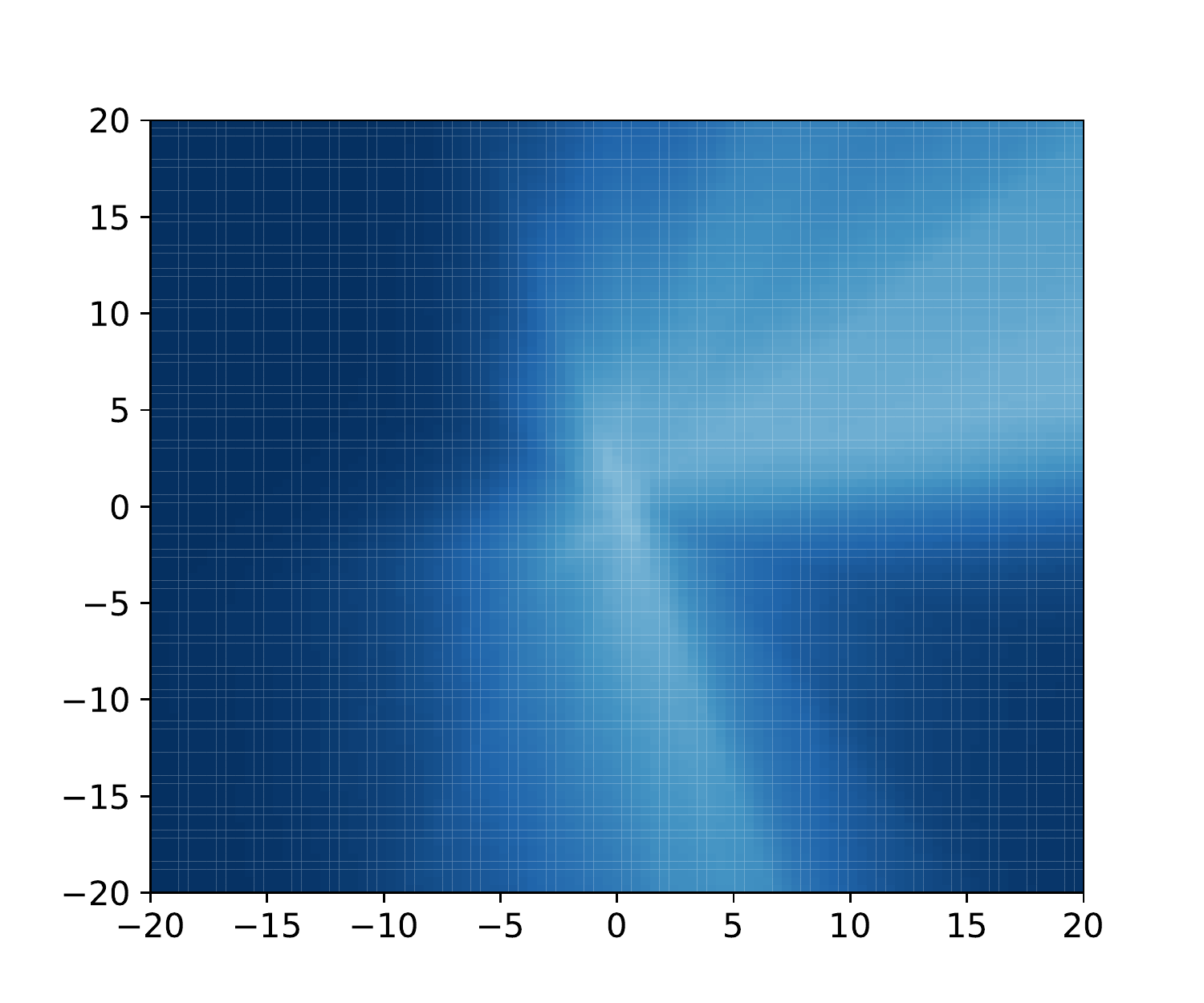}
    \caption{Epistemic Uncertainty estimated by SDE-Net.}
    \label{fig:un}
  \end{subfigure}
  \caption{Visualization of the epistemic uncertainty estimated by SDE-Net (darker colors represent higher uncertainties in the heat map). }
  \label{fig:vis}
\end{figure}

\subsection{Expected Calibration Error}
In this subsection, we measure the expected calibration error (ECE, \citep{Guo:2017}) to see if the confidences produced by the models are trustworthy.
Fig.~\ref{fig:ece} shows the ECE of each method on MNIST and SVHN. On MNIST, SDE-Net can achieve competitive results compared with DeepEnsemble and MC-dropout and outperforms other methods. On SVHN, SDE-Net outperforms all the baselines.

\begin{figure}[h!]
  \centering
  \begin{subfigure}[b]{0.4\linewidth}
    \includegraphics[width=\linewidth]{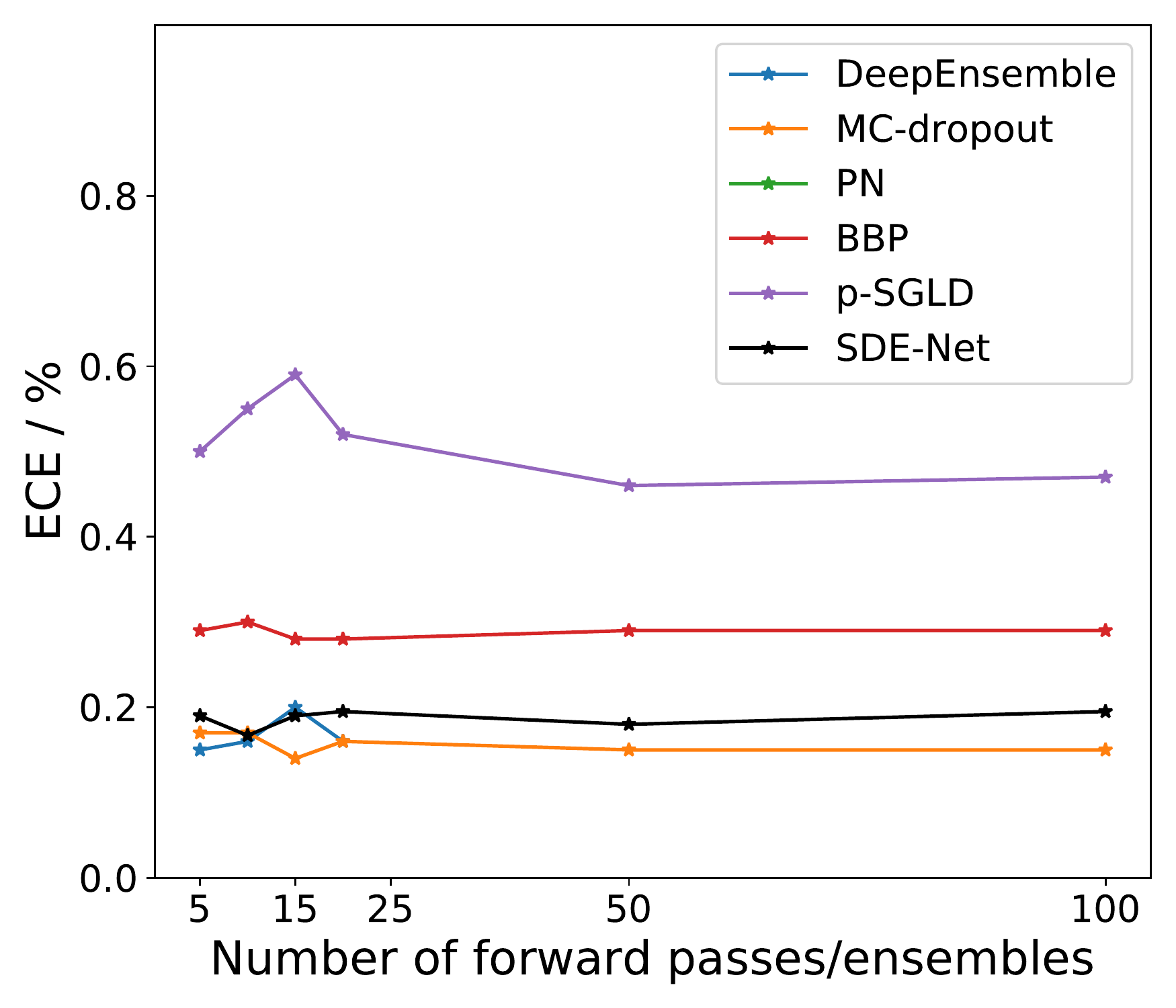}
    \caption{MNIST}
  \end{subfigure}
  \hspace{0.06\linewidth}
  \begin{subfigure}[b]{0.4\linewidth}
    \includegraphics[width=\linewidth]{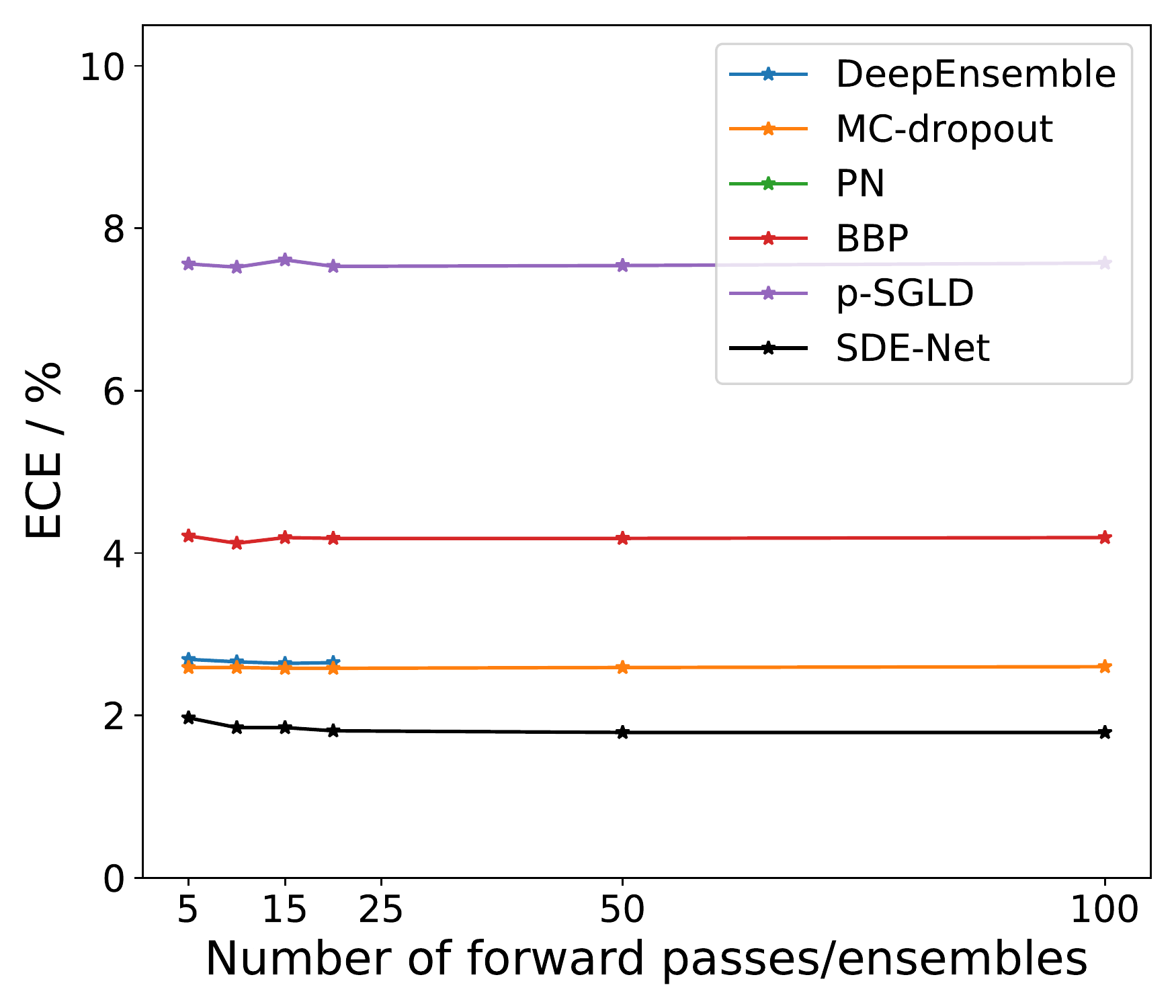}
    \caption{SVHN}
  \end{subfigure}
  \caption{Expected calibration error (ECE) vs number of forward passes/ensembles. PN is outside the range and not shown}
  \label{fig:ece}
\end{figure}

\subsection{Ablation Study}

\textbf{Robustness to different pseudo OOD data.}
In this set of experiments, we report additional experimental 
results for OOD detection in classification tasks. 
We use MNIST as the in-distribution training dataset, and explore using other data sources as OOD data beyond using in-distribution data perturbed by Gaussian noise. The results are shown in  Table \ref{OOD}. As we can see, the performance of PN is very poor when using Gaussian noise and training data perturbed by Gaussian noise. When using SVHN as OOD data during training, its performance is good. This suggests that PN is easy to be overfitted by the OOD data used in training. Our SDE-Net can achieve good performance in all settings, which shows its superior robustness.

\begin{table}[h!]
  \caption{Additional Results for OOD detection. MNIST is used as in-distribution training data. The OOD data used during training is in the bracket beside each model. Gaussian means directly sampling from $\mathcal{N}(0,1)$ as pseudo OOD data. Training+Gaussian means perturbing training data by Gaussian noise (0 mean and variance 4) as pseudo OOD data.
  SVHN means directly use the training set of SVHN as pseudo OOD data.
  We report the average performance and standard deviation for 5 random initializations.}
  \label{OOD}
  \centering
  \begin{tabular}{ccccccc}
  \hline
    \toprule
  OOD Data (test) & Model & \makecell{TNR \\at TPR $95\%$} & AUROC& \makecell{Detection \\accuracy} & \makecell{AUPR \\in} & \makecell{AUPR\\ out}\\
    \midrule
    \multirow{6}{*}{SVHN}&SDE-Net(SVHN)&$99.9\pm 0.0$ &$99.9 \pm 0.0 $ &$99.8 \pm 0.1$&$99.9 \pm 0.0 $&$99.9 \pm 0.0$  \\
    &SDE-Net(Gaussian)  & $99.4 \pm 0.1$ & $99.9 \pm0.0$ & $98.5 \pm 0.2$ & $99.7 \pm 0.1$ & $100.0 \pm 0.0$ \\
    &SDE-Net(training+Gaussian)  & $97.8\pm 1.1$ & $99.5\pm 0.2$ & $97.0\pm 0.2$ & $98.6 \pm 0.6$ & $99.8 \pm 0.1$ \\
    \cline{2-7}
    &PN(SVHN)  & $100.0\pm 0.0$ & $100.0\pm 0.0$ & $100.0\pm 0.0$ & $100.0\pm 0.0$ & $100.0\pm 0.0$ \\
      &PN(Gaussian)& $89.0\pm 2.9$ & $92.9 \pm 1.2$ & $92.3\pm 2.2$ & $68.1 \pm 6.5$ & $97.6\pm 0.7$ \\
   &PN(training+Gaussian)& $90.4 \pm 2.8$ & $94.1 \pm 2.2$ & $93.0\pm 1.4$ & $73.2\pm 7.3$ & $98.0\pm0.6$ \\
    \hline
    \multirow{6}{*}{SEMEION} 
    &SDE-Net(SVHN) &$100.0\pm 0.0$ &$99.9\pm 0.0$& $99.9\pm 0.0$&$100.0\pm 0.0$ &$99.0\pm 0.2$ \\    
    &SDE-Net(Gaussian)  & $99.9\pm 0.1$ & $100.0 \pm 0.0$ & $99.0 \pm 0.3$ & $100.0 \pm 0.0$ & $99.8 \pm 0.1$ \\
     &SDE-Net(training+Gaussian)  & $99.6\pm0.2$ & $99.9\pm0.1$ & $98.6\pm0.5$ & $100.0\pm0.0$ & $99.5\pm 0.3$ \\
     \cline{2-7}
     &PN(SVHN)  & $98.0\pm 0.8$ & $98.7\pm 0.3$ & $97.3\pm 1.2$ & $99.6\pm 0.1$ & $95.7\pm 2.3$ \\
      &PN(Gaussian)& $91.0\pm 2.3$ & $94.9\pm 2.6$ & $93.2\pm 1.5$ & $97.8\pm 0.6$ & $86.5\pm 3.5$ \\
   &PN(training+Gaussian)& $93.4\pm2.2$ & $96.1\pm 1.2$ & $94.5\pm 1.1$ & $98.4\pm0.7$ & $88.5\pm 1.3$ \\
    \hline
    \multirow{6}{*}{CIFAR10} 
    &SDE-Net(SVHN)&$100.0\pm 0.0$ &$99.9 \pm 0.0$& $99.7\pm 0.1$ &$99.9\pm 0.1$ &$99.8\pm 0.1$ \\
    &SDE-Net(Gaussian)  & $99.8\pm 0.1$ & $100.0\pm 0.0$ & $98.9 \pm 0.4$ & $100.0\pm 0.0$ & $100.0
    \pm 0.0$ \\
      &SDE-Net(training+Gaussian)  & $99.7\pm 0.2$ & $99.9\pm 0.0$ & $98.3\pm 0.4$ & $99.9\pm 0.0$ & $99.9\pm 0.0$ \\
      \cline{2-7}
        &PN(SVHN)  & $100.0 \pm 0.0$ & $100.0\pm 0.0$ & $99.8\pm 0.1$ & $100.0\pm 0.0$ & $100.0\pm 0.0$ \\
           &PN(Gaussian)& $96.8 \pm 1.2$ & $97.7\pm 0.7$ & $96.5 
           \pm 0.6$ & $94.3 \pm 1.2$ & $98.2 \pm 0.3$ \\
        &PN(training+Gaussian)& $97.6 \pm 0.7$ & $98.3 \pm 0.8$ & $97.0 \pm 1.2$ & $96.0\pm 1.7$ & $97.3 \pm 1.2$\\
    \bottomrule
    \hline
  \end{tabular}
\end{table}

\begin{table}[!h]
  \caption{Classification and out-of-distribution detection results on MNIST and SVHN. All values are in percentage, and larger values indicates better detection performance. We report the average performance and standard deviation for 5 random initializations. }
    \label{table:ablation:ood}
    \centering
    \begin{tabular}{cccccccccc}
      \toprule
  ID &  OOD & Model & \makecell{TNR \\at TPR $95\%$} & AUROC& \makecell{Detection \\accuracy} & \makecell{AUPR \\in} & \makecell{AUPR\\ out}\\
      \midrule
     \multirow{2}{*}{MNIST}& \multirow{2}{*}{SEMEION}
     &SDE-Net w.o. reg &$93.7\pm 1.1$ & $97.9\pm 0.4$ &$95.2\pm 0.9 $ & $99.8\pm 0.1$ & $89.8 \pm 1.2$\\
       &&SDE-Net &  $\bm{99.6}\pm 0.2$ & $\bm{99.9}\pm 0.1$ & $\bm{98.6}\pm 0.5$ & $\bm{100.0}\pm 0$ & $\bm{99.5}\pm 0.3$ \\
      \hline
      \multirow{2}{*}{MNIST}&\multirow{2}{*}{SVHN} 
      &SDE-Net w.o. reg & $90.3 \pm 1.3$ &$96.6 \pm 1.3$&$92.2 \pm 1.2$ & $90.0 \pm 2.2$ &  $98.2 \pm 0.4$\\
        &&SDE-Net  &  $\bm{97.8}\pm 1.1$ & $\bm{99.5}\pm 0.2$ & $\bm{97.0}\pm 0.2$ & $\bm{98.6}\pm 0.6$ & $\bm{99.8}\pm 0.1$ \\
        \hline
          \multirow{2}{*}{SVHN}&  \multirow{2}{*}{CIFAR10}
          &SDE-Net w.o. reg& $68.2 \pm 2.4$ & $93.9 \pm 0.7$&$90.3 \pm 0.9$&$97.2 \pm 0.7$ &$85.2 \pm 1.2$ \\
      &&SDE-Net  & $\bm{87.5}\pm 2.8$ & $\bm{97.8}\pm 0.4$ & $\bm{92.7}\pm 0.7$ & $\bm{99.2}\pm 0.2$ & $\bm{93.7}\pm 0.9$ \\
      \hline
      \multirow{2}{*}{SVHN}&\multirow{2}{*}{CIFAR100} &SDE-Net w.o. reg&$65.2 \pm 1.3$ &$92.9 \pm 0.9$&$88.7 \pm 0.6$& $97.2 \pm 0.3$ & $83.4 \pm 0.7$  \\
        &&SDE-Net & $\bm{83.4}\pm 3.6$& $\bm{97.0}\pm 0.4$ & $\bm{91.6}\pm 0.7$& $\bm{98.8}\pm 0.1$& $\bm{92.3}\pm 1.1$  \\
      \bottomrule
    \end{tabular}
\end{table}

\textbf{Is the OOD regularizer necessary?}
Our loss objective includes an OOD regularization term which allows us to explicitly train the 
epistemic uncertainty for each data point. This regularizer can be interpreted as our parameter belief from the data space. That is we want our model to give uncertain outputs for OOD data. 
To verify the necessity of this regularization term, we test the uncertainty estimates of SDE-Net trained without the regularizer. As we can see from Table.~\ref{table:ablation:ood}, 
the performance of SDE-Net deteriorates to the same level of traditional NNs without the regularizer term. In Bayesian neural network, the principle of Bayesian inference implicitly enables larger uncertainty in the region that lacks training data. 
Such inference can be costly and we choose to view the DNNs as stochastic dynamic systems. The benefit of such design is that we can directly model the epistemic uncertainty 
level for each data point by the variance of the Brownian motion.

\subsection{Full Results of Table. 2 and Table. 3 of the main paper }
Table.~\ref{table:OOD-regression} shows the full results of Table. 2 of the main paper.

Table.~\ref{table:mis} shows the full results of Table. 3 of the main paper.
\begin{table}[!h]
  \caption{Out-of-distribution detection for regression on Year Prediction MSD + Boston Housing. We report the average performance and standard deviation for 5 random initializations.}
      \label{table:OOD-regression}
      \centering
      \begin{tabular}{cccccccc}
      \toprule
      Model & \# Parameters &RMSE & \makecell{TNR \\at TPR $95\%$} & AUROC & \makecell{Detection \\accuracy}&\makecell{AUPR \\ in}&\makecell{AUPR\\ out}\\
      \midrule
      DeepEnsemble&14.9K$\times 5$ &$\bm{8.6} \pm$ NA& $10.9 \pm$ NA & $59.8\pm$ NA& $61.4 \pm $NA & $99.3\pm $NA &$1.3\pm$ NA    \\
      MC-dropout &14.9K & $8.7 \pm 0.0$& $9.6 \pm 0.4$ & $53.0\pm 1.2$& $55.6\pm 1.2$  & $99.2\pm 0.1$  &$1.1\pm 0.1$ \\
      BBP &30.0K &$9.5 \pm 0.2$ &$8.7 \pm 1.5$ & $56.8\pm 0.9$ & $58.3\pm 2.1$ &$99.0\pm 0.0$ &$1.3\pm 0.1$  \\
      p-SGLD &14.9K & $9.3 \pm 0.1$& $9.2\pm 1.5$ &$52.3\pm 0.7$ & $57.3\pm 1.9$ & $99.4 \pm 0.0$ & $1.1\pm 0.2$ \\
      SDE-Net & 12.4K& $8.7 \pm 0.1$& $\bm{60.4}\pm 3.7$ & $\bm{84.4}\pm 1.0$ & $\bm{80.0} \pm 0.9$ & $\bm{99.7} 
      \pm 0.0$& $\bm{21.3}\pm 4.1$ \\
      \bottomrule
      \end{tabular}
      \vspace{-3ex}
  \end{table}

  \begin{table}[!h]
    \caption{Misclassification detection performance on MNIST and SVHN. We report the average performance and standard deviation for 5 random initializations.}
      \label{table:mis}
      \centering
      \begin{tabular}{cccccccc}
        \toprule
     Data & Model & \makecell{TNR \\at TPR $95\%$} &AUROC & \makecell{Detection \\accuracy}&\makecell{AUPR \\succ} & \makecell{AUPR\\ err}\\
        \midrule
     \multirow{4}{*}{MNIST}
     &Threshold  & $85.4\pm 2.8$ & $94.3 \pm 0.9$& $92.1 \pm 1.5$ &$ 99.8 \pm 0.1$ &$31.9 \pm 8.3$ & \\
     &  DeepEnsemble& $89.6 \pm $NA  &$\bm{97.5} \pm$ NA&  $93.2 \pm $NA &$\bm{100.0}\pm$ NA&$41.4\pm$ NA    \\
            &MC-dropout&  $85.4 \pm 4.5$ &$95.8 \pm 1.3$& $91.5 \pm 2.2$  &$99.9 \pm 0.0$&$33.0\pm 6.7$  \\
        &PN & $85.4 \pm 2.8$& $91.8\pm 0.7 $&$91.0 \pm 1.1$ &$99.8\pm0.0$&$33.4 \pm 4.6$\\
            &BBP & $88.7\pm 0.9$ & $96.5 \pm 2.1$  &$93.1\pm 0.5$ & $\bm{100.0}\pm 0.0$ & $35.4\pm 3.2$\\
        &P-SGLD & $\bm{93.2} \pm 2.5$& $96.4\pm 1.7$ &$\bm{98.4}\pm 0.2$ &$\bm{100.0\pm}0.0$ & $\bm{42.0\pm}2.4$\\
        &SDE-Net&$88.5 \pm 1.3$ &$96.8\pm 0.9$ & $92.9 \pm 0.8$&$\bm{100.0 \pm} 0.0$&$36.6 \pm 4.6$   \\
        \hline
         \multirow{4}{*}{SVHN} 
         &Threshold  & $66.4 \pm 1.7$  &$90.1\pm 0.3$& $85.9 \pm 0.4$ &$99.3\pm 0.0$ & $42.8\pm 0.6$& \\
         &  DeepEnsemble& $\bm{67.2} \pm $NA &$91.0\pm$ NA& $86.6 \pm$ NA &$\bm{99.4}\pm$ NA &$46.5\pm$ NA    \\
            &MC-dropout& $65.3 \pm 0.4$ &$90.4\pm 0.6$&$85.5 \pm 0.6$ &$99.3\pm 0.0$ &$45.0\pm 1.2$  \\
        &PN &$64.5 \pm 0.7$ & $84.0\pm 0.4$ &$81.5 \pm 0.2$ &$98.2\pm 0.2$ & $43.9\pm 1.1$\\
        &BBP & $58.7 \pm 2.1$ &$91.8 \pm 0.2$  & $85.6 \pm 0.7$  & $99.1\pm 0.1$ & $50.7\pm 0.9 $\\
        &P-SGLD& $64.2 \pm 1.3 $ & $\bm{93.0}\pm 0.4$ & $\bm{87.1} \pm 0.4$ &$\bm{99.4}\pm 0.1$ & $48.6 \pm 1.8$\\
        &SDE-Net& $65.5 \pm 1.9$ &$92.3\pm 0.5$& $86.8 \pm 0.4$&$\bm{99.4\pm}0.0$&$\bm{53.9}\pm 2.5$  \\
        \bottomrule
      \end{tabular}
  \end{table}

  \newpage
  \newpage
  \newpage
  \section{Network Architecture}
  \subsection{Classification Task}
  Downsampling layer:
  
  \begin{python}
          self.downsampling_layers = nn.Sequential(
          #change the in planes to 3 for SVHN
              nn.Conv2d(1, dim, 3, 1), 
              norm(dim),
              nn.ReLU(inplace=True),
              nn.Conv2d(dim, dim, 4, 2, 1),
              norm(dim),
              nn.ReLU(inplace=True),
              nn.Conv2d(dim, dim, 4, 2, 1),
          )
  \end{python}
  Drift neural network:
  \begin{python}
  class Drift(nn.Module):
      def __init__(self, dim):
          super(Drift, self).__init__()
          self.norm1 = norm(dim)
          self.relu = nn.ReLU(inplace=True)
          self.conv1 = ConcatConv2d(dim, dim, 3, 1, 1)
          self.norm2 = norm(dim)
          self.conv2 = ConcatConv2d(dim, dim, 3, 1, 1)
          self.norm3 = norm(dim)
  
      def forward(self, t, x):
          out = self.norm1(x)
          out = self.relu(out)
          out = self.conv1(t, out)
          out = self.norm2(out)
          out = self.relu(out)
          out = self.conv2(t, out)
          out = self.norm3(out)
          return out 
  \end{python}
  
  Diffussion neural network for MNIST:
  
  \begin{python}
  class Diffusion(nn.Module):
      def __init__(self, dim_in, dim_out):
          super(Diffusion, self).__init__()
          self.norm1 = norm(dim_in)
          self.relu = nn.ReLU(inplace=True)
          self.conv1 = ConcatConv2d(dim_in, dim_out, 3, 1, 1)
          self.norm2 = norm(dim_in)
          self.conv2 = ConcatConv2d(dim_in, dim_out, 3, 1, 1)
          self.fc = nn.Sequential(norm(dim_out), nn.ReLU(inplace=True), nn.AdaptiveAvgPool2d((1, 1)), Flatten(), nn.Linear(dim_out, 1), nn.Sigmoid())
      def forward(self, t, x):
          out = self.norm1(x)
          out = self.relu(out)
          out = self.conv1(t, out)
          out = self.norm2(out)
          out = self.relu(out)
          out = self.conv2(t, out)
          out = self.fc(out)
          return out
  \end{python}
  
  Diffusion network for SVHN:
  
  \begin{python}
  class Diffusion(nn.Module):
      def __init__(self, dim_in, dim_out):
          super(Diffusion, self).__init__()
          self.norm1 = norm(dim_in)
          self.relu = nn.ReLU(inplace=True)
          self.conv1 = ConcatConv2d(dim_in, dim_out, 3, 1, 1)
          self.norm2 = norm(dim_in)
          self.conv2 = ConcatConv2d(dim_in, dim_out, 3, 1, 1)
          self.norm3 = norm(dim_in)
          self.conv3 = ConcatConv2d(dim_in, dim_out, 3, 1, 1)
          self.fc = nn.Sequential(norm(dim_out), nn.ReLU(inplace=True), nn.AdaptiveAvgPool2d((1, 1)), Flatten(), nn.Linear(dim_out, 1), nn.Sigmoid())
      def forward(self, t, x):
          out = self.norm1(x)
          out = self.relu(out)
          out = self.conv1(t, out)
          out = self.norm2(out)
          out = self.relu(out)
          out = self.conv2(t, out)
          out = self.norm3(out)
          out = self.relu(out)
          out = self.conv3(t, out)
          out = self.fc(out)
          return out
  \end{python}

  ResNet block architecture:
  
  \begin{python}
  class ResBlock(nn.Module):
      expansion = 1
      def __init__(self, inplanes, planes, stride=1, downsample=None):
          super(ResBlock, self).__init__()
          self.norm1 = norm(inplanes)
          self.relu = nn.ReLU(inplace=True)
          self.downsample = downsample
          self.conv1 = conv3x3(inplanes, planes, stride)
          self.norm2 = norm(planes)
          self.conv2 = conv3x3(planes, planes)
  
      def forward(self, x):
          shortcut = x
  
          out = self.relu(self.norm1(x))
  
          if self.downsample is not None:
              shortcut = self.downsample(out)
  
          out = self.conv1(out)
          out = self.norm2(out)
          out = self.relu(out)
          out = self.conv2(out)
  
          return out + shortcut
  \end{python}
  
  For BBP, we use an identical Residue block architecture and a fully factorised Gaussian approximate posterior on the weights.

  \subsection{Regression Task} 
  The network architecture for DeepEnsemble, MC-dropout and p-SGLD:
  \begin{python}
  class DNN(nn.Module):
      def __init__(self):
          super(DNN, self).__init__()
          self.fc1 = nn.Linear(90, 50)
          self.dropout1 = nn.Dropout(0.5)
          self.fc2 = nn.Linear(50, 50)
          self.dropout2 = nn.Dropout(0.5)
          self.fc3 = nn.Linear(50, 50)
          self.dropout3 = nn.Dropout(0.5)
          self.fc4 = nn.Linear(50, 50)
          self.dropout4 = nn.Dropout(0.5)
          self.fc5 = nn.Linear(50, 50)
          self.dropout5 = nn.Dropout(0.5)
          self.fc6 = nn.Linear(50, 2)
  
      def forward(self, x):
          x = self.dropout1(F.relu(self.fc1(x)))
          x = self.dropout2(F.relu(self.fc2(x)))
          x = self.dropout3(F.relu(self.fc3(x)))
          x = self.dropout4(F.relu(self.fc4(x)))
          x = self.dropout5(F.relu(self.fc5(x)))
          x = self.fc6(x)
          mean = x[:,0]
          sigma = F.softplus(x[:,1])+1e-3
          return mean, sigma
  \end{python}
  
  For BBP, we use an identical architecture with a fully factorised Gaussian approximate posterior on the weights.

  For SDE-Net:
  
  Drift neural network:
  \begin{python}
  class Drift(nn.Module):
      def __init__(self):
          super(Drift, self).__init__()
          self.fc = nn.Linear(50, 50)
          self.relu = nn.ReLU(inplace=True)
      def forward(self, t, x):
          out = self.relu(self.fc(x))
          return out 
  \end{python}
  
  Diffusion neural network:
  \begin{python}
  class Diffusion(nn.Module):
      def __init__(self):
          super(Diffusion, self).__init__()
          self.relu = nn.ReLU(inplace=True)
          self.fc1 = nn.Linear(50, 100)
          self.fc2 = nn.Linear(100, 1)
      def forward(self, t, x):
          out = self.relu(self.fc1(x))
          out = self.fc2(out)
          out = F.sigmoid(out)
          return out
  \end{python}



\end{document}